\pdfoutput=1
\documentclass{article}
\usepackage{ijcai24}

\usepackage{algorithm}
\usepackage{algorithmic}
\usepackage{amssymb,amsfonts}       % blackboard math symbols
\usepackage{times}
\usepackage{soul}
\usepackage{url}
\usepackage[hidelinks]{hyperref}
\usepackage[utf8]{inputenc}
\usepackage[small]{caption}
\usepackage{subcaption}
\usepackage{graphicx}
\usepackage{latexsym}
\usepackage{amsmath}
\usepackage{amsthm}
\usepackage{mathtools}
\usepackage{booktabs}
\usepackage{algorithm}
\usepackage{algorithmic}
\usepackage{enumitem}
\usepackage{float}
\usepackage[switch]{lineno}
\usepackage{multirow}
\usepackage{diagbox}
\usepackage{booktabs}
\usepackage{wrapfig}
\usepackage{xcolor}

\newtheorem{theorem}{Theorem}
\newtheorem{lemma}{Lemma}
\newtheorem{definition}{Definition}

\newtheorem{remark}{Remark}
\newtheorem{claim}{Claim}

\title{Approximate Algorithms For $k$-Sparse Wasserstein Barycenter With Outliers}

\author{
Qingyuan Yang \and
Hu Ding  \thanks{Corresponding author.}
\affiliations
School of Computer Science and Technology\\
University of Science and Technology of China
\emails
yangqingyuan@mail.ustc.edu.cn, 
huding@ustc.edu.cn
}
\begin{document}

\maketitle

\begin{abstract}
Wasserstein Barycenter (WB) is one of the most fundamental optimization problems in  optimal transportation. Given a set of distributions, the goal of WB is to find a new distribution that minimizes the average Wasserstein distance to them. The problem becomes even harder if we restrict the solution to be ``$k$-sparse''. In this paper, we study the $k$-sparse WB problem in the presence of outliers, which is a more practical setting since real-world data  often contains noise. Existing WB algorithms  cannot be directly extended to handle the case with outliers, and thus it is urgently needed to develop some novel ideas. First, we investigate the relation between $k$-sparse WB with outliers and the clustering (with outliers) problems. In particular, we propose a clustering based LP method that yields constant approximation factor for the $k$-sparse WB with outliers  problem. Further, we utilize the coreset technique to achieve the $(1+\epsilon)$-approximation factor for any $\epsilon>0$, if the dimensionality is not high. Finally, we conduct the experiments for our proposed algorithms and illustrate their efficiencies in practice. 
\end{abstract}

\section{Introduction}
\label{sec-intro}
Let $P=\{p_1, p_2, \cdots, p_{n_1}\}$ and $Q=\{q_1, q_2, \cdots, q_{n_2}\}$ be two sets of weighted points in the Euclidean space $\mathbb{R}^d$, where we use $w_P(\cdot)$ ({\em resp.,} $w_Q(\cdot)$) to denote the non-negative weight function for each point of $P$ ({\em resp.,} $Q$); we also assume that $P$ and $Q$ have the same total weight, {\em i.e.,} $\sum^{n_1}_{i=1}w_P(p_i)=\sum^{n_2}_{j=1}w_Q(q_j)=n>0$. For any $l\geq 1$, the seminal  \textbf{Wasserstein Distance}~\cite{villani2021topics} is to measure their minimum transportation cost: 
\vspace{-5pt}
\begin{eqnarray}
\mathcal{W}(P, Q)=\min_{F} \Big(\sum^{n_1}_{i=1}\sum^{n_2}_{j=1}f_{ij}||p_i-q_j||^l\Big)^{\frac{1}{l}}, \label{for-wd}
\end{eqnarray}
where $||\cdot||$ is the Euclidean distance, and the flow set $F=\{f_{ij}\mid 1\leq i\leq n_1, 1\leq j\leq n_2\}$ from $P$ to $Q$ should satisfy: $\sum^{n_1}_{i=1} f_{ij}=w_Q(q_j)$ for any $j$, and $\sum^{n_2}_{j=1} f_{ij}=w_P(p_i)$ for any $i$. We usually set $l=1$ or $2$ for most practical applications. Also note that ``$||p_i-q_j||^l$'' in (\ref{for-wd}) can be  replaced by other distance function if the input $P$ and $Q$ are in some other metric space rather than the Euclidean space. The Wasserstein distance is one of the most fundamental topics in mathematics in past decades, and recently finds a broad range of applications in machine learning, such as image retrieval~\cite{rubner2000earth}, generative model~\cite{DBLP:conf/icml/ArjovskyCB17}, and robust optimization~\cite{kuhn2019wasserstein}. 

In this paper, we focus on an important optimization problem from Wasserstein distance, \textbf{$k$-sparse Wasserstein Barycenter (WB)}~\cite{DBLP:journals/jco/BorgwardtP21}. Suppose the input contains $m\geq 1$ weighted point sets $P_1, P_2, \cdots, P_m$, the problem of $k$-sparse WB for a given $k\in \mathbb{Z}^+$ is to construct a new weighted point set $S$ with the support size $|\mathtt{supp}(S)|=k$, such that 
\vspace{-5pt}
\begin{eqnarray}
\frac{1}{m}\sum^m_{j=1}\mathcal{W}^l(S, P_j) \label{for-kwb}
\end{eqnarray}
is minimized. The Wasserstein barycenter~\cite{agueh2011barycenters} is a natural representation for the average of a set of given distributions. Recently, it has been applied to a number of real-world problems, such as medical imaging~\cite{DBLP:conf/ipmi/GramfortPC15}, Bayesian learning~\cite{srivastava2018scalable}, clustering~\cite{ho2017multilevel},  and natural language processing~\cite{singh2020context}. In practice, we often prefer a simple representation of the barycenter with low complexity, and thus we have the sparsity requirement  ``$|\mathtt{supp}(S)|=k$''.  

Note that it is quite challenging to achieve a high-quality solution with theoretical guarantee  for the objective (\ref{for-kwb}); this is  mainly due to two aspects, where one is from the inherent hardness for computing the Wasserstein distance, and the other reason is due to the restriction of ``$|\mathtt{supp}(S)|=k$'' which is a troublesome combinatorial constraint for the optimization. For example, \cite{DBLP:journals/jco/BorgwardtP21} proved that the $k$-sparse WB problem is NP-hard, even if $m=3$ and the dimensionality $d=2$; on the other hand, if the $k$-sparse restriction is removed, one can compute Wasserstein barycenter in polynomial time in low dimensional space~\cite{DBLP:journals/jmlr/AltschulerB21}. 

We consider an even harder but also more practical variant of the $k$-sparse WB problem called ``\textbf{$k$-sparse WB with outliers}''. Roughly speaking, we allow a certain fraction of outliers for the mapping from each $P_j$ to $S$ in (\ref{for-kwb}) (the formal definition is shown in Section~\ref{sec-pre}). The motivation of allowing outliers is very natural in practical scenarios. Suppose we want to compute the WB for a set of images~\cite{cuturi2014fast}; it is common that the input images may contain noises and/or some irrelevant objects. To achieve a more robust solution, it is better to compute a barycenter that matches each input image partially (the un-matched part can be viewed as outliers).

Actually, the study on Wasserstein distance with outliers  has already attracted attentions  recently in machine learning~\cite{chapel2020partial,mukherjee2021outlier,le2021robust,nietert2022outlier}. However, the existing algorithms for computing WB with outliers are still quite limited, to the best of our knowledge. Though several robust models and algorithms for WB  have been studied~\cite{DBLP:conf/nips/CazellesTF21,le2021robust}, their methods do not explicitly handle outliers.  In the current article, we consider developing approximate algorithms for $k$-sparse WB with outliers. For the sake of simplicity, we fix $l=2$ for the Wasserstein distance in (\ref{for-wd}); actually our results can be easily extended for any $l\geq 1$. 
Our main ideas rely on the vanilla clustering algorithms.  Suppose we have an $\alpha$-approximation algorithm $\mathcal{A}$ for $k$-means clustering ({\em e.g.,} the algorithms from~\cite{kanungo2002local,Arthur2007kmeansTA}) and a $\beta$-approximation algorithm $\mathcal{B}$ for $k$-means clustering with outliers ({\em e.g.,} the algorithms from~\cite{friggstad2019approximation,gupta2017local}), where $\alpha, \beta\geq 1$. \textbf{Our contributions are as follows.}
\begin{itemize}
\item Our first contribution is to illustrate the relation between $k$-sparse WB with outliers and the problems of $k$-means clustering and $k$-means clustering with outliers. We are aware that the relation between the vanilla WB (without outliers) problem and $k$-means clustering has been discussed before~\cite{cuturi2014fast}. However, it is much more challenging to analyze the case with outliers; for example, even one single outlier can seriously destroy the clustering performance. In particular, we need to develop some significantly new insight to investigate the influence of outliers in the context of the complicated Wasserstein flows. We show that one can achieve an $O(\alpha)$-approximate solution by utilizing the algorithm $\mathcal{A}$, but the returned barycenter has the support size larger than $k$. 
If we want to keep the support size being equal to $k$, we can take advantage of $\mathcal{B}$ instead and achieve an $O(\beta)$-approximate solution. 

\item We further study the problem in low-dimensional space ({\em e.g.,} computing the WB for a set of 2D images). Our idea follows the aforementioned clustering method, but in a more  sophisticated  manner. We utilize the low-dimensional coreset technique~\cite{DBLP:conf/stoc/Har-PeledM04} to generate a set of ``anchor'' points, and then build a set of non-uniform grids surrounding them. For any $\epsilon>0$, if we relax the ``$k$-sparse'' requirement, we can compute a WB that achieves a $(1+\epsilon)$ approximation factor based on those grids. Moreover, our result can be generalized to any metric space with constant doubling dimension ({\em e.g.,} the input distributions could have small intrinsic dimension even in high-dimensional space~\cite{roweis2000nonlinear}).  
\end{itemize}

\begin{remark}
\label{rem-result1}
The  algorithms $\mathcal{A}$ and $\mathcal{B}$ can be bi-criteria approximation algorithms, that is, they can return more than $k$ cluster centers (if we relax the $k$-sparse requirement). The benefit of using bi-criteria approximation algorithms  is that they often achieve lower $\alpha$ and $\beta$. For example, the popular $k$-means++ algorithm yields an $O(\log k)$ approximation factor~\cite{Arthur2007kmeansTA}; but if we run the $k$-means++ seeding procedure more than $k$ steps (say $\lambda k$ with some constant integer $\lambda>1$), the approximation factor can be reduced to be $O(1)$~\cite{aggarwal2009adaptive}. 
\end{remark}

\subsection{Related Works}
\label{sec-relate}

\textbf{Wasserstein distance.} The research on computing the Wasserstein distance (\ref{for-wd}) has gained a great amount of attentions in theory and various practical applications. It is easy to see that (\ref{for-wd}) can be viewed as a min-cost flow problem that one can apply the network simplex algorithm to solve it~\cite{ahuja1988network}.
In machine learning community, \cite{DBLP:conf/nips/Cuturi13} proposed a new variant  called ``Sinkhorn distance'' that can be computed much faster than the original Wasserstein distance. Following Cuturi's work,  \cite{DBLP:conf/nips/AltschulerWR17} proposed a nearly-linear time Wasserstein distance algorithm.  \cite{DBLP:conf/nips/GenevayCPB16} proposed a stochastic algorithm for solving large-scale optimal transportation. Some recent improvements also include~\cite{dvurechensky2018computational,lin2019efficient}. 
As mentioned before, several works on the Wasserstein distance with outliers problem and its applications ({\em e.g.,} outlier detection and shape matching) were also proposed recently~\cite{chapel2020partial,mukherjee2021outlier,le2021robust,nietert2022outlier}.

\textbf{Wasserstein barycenter.} The study of Wasserstein barycenter mainly focuses on  two different types.   One is called ``\textbf{fixed-support WB}'', where the barycenter $S$ in  (\ref{for-kwb}) has a fixed support (the support size can be very large) and the task is to determine the weight  distribution over the support such that the average Wasserstein distance to the given $m$ input distributions is minimized. The other one is called ``\textbf{free-support WB}'' where the  barycenter $S$ can have the support that locates anywhere in the space.  The former problem is relatively easier to solve, since the weight distribution can be obtained by computing a linear programming (LP)~\cite{auricchio2019computing}. A number of efficient algorithms for fixed-support WB have been proposed.
%~\cite{cuturi2014fast,cuturi2016smoothed,claici2018stochastic,kroshnin2019complexity,ge2019interior,lin2020fixed}. 
For example, \cite{claici2018stochastic} presented a stochastic algorithm for WB; \cite{ge2019interior} developed a novel interior-point method by removing redundant constraints for the LP; \cite{lin2020fixed} provided a fast iterative Bregman projection algorithm.

On the other hand, the  free-support WB problem is more challenging. Recently, \cite{altschuler2022wasserstein} showed that it is NP-hard to compute even a  WB with $\epsilon$ additive error in Euclidean space (if $d$ is not constant); only for low-dimensional space, one can obtain the optimal WB in polynomial time~\cite{DBLP:journals/jmlr/AltschulerB21}.  \cite{borgwardt2022lp} provided a $2$-approximate WB in Euclidean space. But those algorithms cannot guarantee small support for the obtained barycenter ({\em e.g.,} the support can be as large as $O(\sum^m_{j=1}|\mathtt{supp}(P_j)|)$). 

If we further require $|\mathtt{supp}(S)|=k$, the problem can be much more challenging; as mentioned before the $k$-sparse WB problem is NP hard even for $m=3$ and $d=2$~\cite{DBLP:journals/jco/BorgwardtP21}. Most existing algorithms for $k$-sparse WB rely on the idea of alternating minimization, that is, they iteratively update the location of the $k$ sparse support of $S$ and the weight distribution, until the solution converges to some local optimum~\cite{cuturi2014fast,ye2017fast,claici2018stochastic,ge2019interior}. 

\section{Preliminaries}
\label{sec-pre}

To formally define the $k$-sparse WB with outliers problem, we should provide the definition for  Wasserstein distance with outliers first (the vanilla Wasserstein distance (\ref{for-wd}) is the special case with zero outlier). For any weighted point set $P\subseteq \mathbb{R}^d$, we use $w_P(p)$ and $w_P(S)$ to denote the non-negative weight of a point $p\in P$ and the total weight of a subset  $S\subset P$, respectively. We also define the relation ``$\preceq$'' between two point sets $P$ and $P'$: $P'\preceq P$ if $\mathtt{supp}(P')=\mathtt{supp}(P)$ and $w_{P'}(p)\leq w_{P}(p)$  for  any $p\in\mathtt{supp}(P)$.   Further, we define  the  set $\mathbb{P}^P_{-z}=\{P'\subset \mathbb{R}^d\mid P'\preceq P \text{ and }w_{P'}(P')=w_{P}(P)-z\}$ for any given non-negative value $z\leq w_P(P)$. 

\begin{definition}[Wasserstein distance with $z$ outliers]
\label{def-fwd}
Let $n>z\geq 0$. Suppose $P=\{p_1, p_2, \cdots, p_{n_1}\}$ and $Q=\{q_1, q_2, \cdots, q_{n_2}\}$ are two sets of  weighted points in   $\mathbb{R}^d$; also $\sum^{n_1}_{i=1}w_P(p_i)=n $ and $ \sum^{n_2}_{j=1}w_Q(q_j)=n-z$.  
The \textbf{Wasserstein distance with $z$ outliers} from $P$ to $Q$ is 
\begin{eqnarray}
\mathcal{W}_{-z}(P, Q)=\min_{P^{\prime}\in \mathbb{P}^P_{-z}} \mathcal{W}(P^{\prime}, Q).\label{for-fwd}
\end{eqnarray} 
The set $P^Q=\arg\min_{P^{\prime}\in \mathbb{P}^P_{-z}} \mathcal{W}(P^{\prime}, Q)$ is called ``\textbf{the inliers of $P$ induced by $Q$}''.
\end{definition}

\begin{remark}
\label{rem-kwoutlier}
The objective function (\ref{for-fwd}) was  similarly defined as ``unbalanced optimal transport'' before~\cite{chapel2020partial,pham2020unbalanced}. 
Also our definition for the function $\mathcal{W}_{-z}(P,Q)$ is one-side, {\em i.e.,} only $P$ contains outliers; actually, a more general definition can be two-side that both  $P$ and $Q$ may contain outliers. We refer the reader to the recent articles~\cite{mukherjee2021outlier,nietert2022outlier} for more detailed discussion. Due to the space limit, we only present our results of one-side here, and leave the results for two-side (which can be easily extended from the one-side results) to our supplement. 
\end{remark}

\textbf{Computing $\mathcal{W}_{-z}(P, Q)$.} Obviously, the total flow $\sum^{n_1}_{i=1}\sum^{n_2}_{j=1} f_{ij}=n-z$ in Definition~\ref{def-fwd}. So the missing flows with the total weight $z$ can be viewed as the outliers. Actually the problem of Wasserstein distance with $z$ outliers can be easily reduced to the vanilla Wasserstein distance problem (\ref{for-wd}) via a ``dummy point'' idea that was studied in~\cite{chapel2020partial,DBLP:journals/corr/abs-2209-02905} before. We add a dummy point $q_*$ to $Q$ with the weight equal to $z$; also we force the ``distance'' between $q_*$ and each $p_i$ to be $0$. Note that we cannot find such a real point $q_*$ in the space, where in reality we just need to set all the entries corresponding to the line of $q_*$ to be $0$ in the $n_1\times (n_2+1)$ distance matrix. Then we can run any off-the-shelf Wasserstein distance algorithm, {\em e.g.,} the simplex network algorithm~\cite{ahuja1988network} or the sinkhorn distance algorithm~\cite{DBLP:conf/nips/Cuturi13}, to compute the solution. Intuitively, the dummy point $q_*$ absorbs the furthest $z$ outliers from $P$. 

\begin{claim}
\label{cl-fwd}
Computing $\mathcal{W}_{-z}(P, Q)$ is equivalent to computing $\mathcal{W}(P, Q\cup\{q_*\})$. 
\end{claim}

\begin{definition}[$k$-sparse WB with $z$ outliers]
\label{def-kwboutlier}
Let $n>z\geq 0$. Suppose the input $\mathbb{P}$ contains $m\geq 1$ weighted point sets $P_1, P_2, \cdots, P_m$ where each $P_j$ has the total weight $n$. Then the problem of $k$-sparse Wasserstein Barycenter with $z$ outliers for a given $k\in \mathbb{Z}^+$ is to construct a new weighted point set $S$ with total weight $n-z$ and the size $|\mathtt{supp}(S)|=k$, such that 
\begin{eqnarray}
\mathtt{Cost}_{-z}(\mathbb{P}, S)=\frac{1}{m}\sum^m_{j=1}\mathcal{W}_{-z}^2(P_j, S) \label{for-kwboutliers}
\end{eqnarray}
is minimized. Throughout this paper, we always use ``$S_{\mathtt{opt}}$'' to denote the optimal solution of (\ref{for-kwboutliers}).  
\end{definition}

\begin{remark}
\label{rem-fixoutlier}
\textbf{(Fixed-support WB with outliers) } As mentioned in Section~\ref{sec-relate}, if we remove the ``$k$-sparse'' requirement and let the support of $S$ be fixed to a  given set $G$ ($|G|$ can be larger than $k$), the problem of WB can be solved by a linear programming~\cite{ge2019interior,lin2020fixed}. Through the idea of Claim~\ref{cl-fwd}, the fixed-support WB with outliers problem can be also solved by using LP. Namely, we add a dummy point $g_*$ to $G$ and set its weight to be $z$; then the problem is exactly equivalent to the vanilla fixed-support WB problem on $G\cup\{g_*\}$. For completeness, we provide the detailed formulation in our supplement. 
\end{remark}

\textbf{Relation to $k$-means clustering with $z$ outliers.} To better illustrate our algorithms for $k$-sparse WB with  outliers, we need to elaborate on its relation to $k$-means clustering with  outliers first.
Given a set $P$ of weighted points in  $\mathbb{R}^d$ with the total weight $n$, the goal of the vanilla $k$-means is to find $k$ cluster centers $C=\{c_1, c_2, \cdots, c_k\}$, such that each point of $P$ is assigned to its nearest cluster center and the total weighted squared distances 
$\mathcal{S}(P,C)=\sum_{p\in P}w_P(p)\cdot \min_{1\leq s\leq k}||p-c_s||^2$  % \label{for-kmeans}
is minimized. If we allow to discard $z$ outliers, the goal becomes to find not only the $k$ cluster centers, but also a set $P^{\prime}\in \mathbb{P}^P_{-z}$, such that $\mathcal{S}(P^{\prime},C)$ is minimized. 
%With a slightly abuse of notations, 
We denote this optimal cost as $\mathtt{Mean}^{k}_{-z}(P)$, and let $\mathtt{C}^{k}_{-z}(P)$ denote the set of optimal cluster centers $\{c_1, c_2, \cdots, c_k\}$ with the weight $w_{\mathtt{C}^{k}_{-z}}(c_s)=$ the total weight of the $s$-th cluster, $1\leq s\leq k$. If $z=0$, we use $\mathtt{Mean}^{k}(P)$ and $\mathtt{C}^{k}(P)$ for simplicity.

\textbf{(1)} First, we consider the basic case $m=1$ for $k$-sparse WB with $z$ outliers. It is easy to see that it is equivalent to the $k$-means clustering with $z$ outliers on $P_1$. 
%Namely, we have the following claim.
\begin{claim}
\label{cl-m1} Suppose $m=1$. 
The set $\mathtt{C}^{k}_{-z}(P_1)$ forms the optimal solution for $k$-sparse WB with $z$ outliers.  Namely,  for any $|\mathtt{supp}(Q)|=k$, {$\mathtt{Mean}^{k}_{-z}(P_1)=\mathcal{W}_{-z}^2(P_1, \mathtt{C}^{k}_{-z}(P_1))\leq\mathcal{W}_{-z}^2(P_1, Q)$. }
\end{claim}
 
\textbf{(2)} Then we consider the general case $m\geq 2$. From Claim~\ref{cl-m1}, we know that the barycenter actually induces $k$ clusters and $z$ outliers on $P_1$ for the case $m=1$. When $m\geq 2$, the $k$ points of the barycenter also induce $k$ clusters on $\cup^m_{j=1}P_j$ but with additional constraint. 
\begin{claim}
\label{cl-m2}
The optimal solution of $k$-sparse WB with $z$ outliers is equivalent to solving the  $k$-means clustering with $mz$ outliers on $\cup^m_{j=1}P_j$ with the following constraint: for each obtained cluster $U_s$, $1\leq s\leq k$, $w_{U_s}(U_s\cap P_1)=w_{U_s}(U_s\cap P_2)=\cdots=w_{U_s}(U_s\cap P_m)$.
\end{claim}
Actually the above constrained $k$-means clustering with outliers problem is a special {\em fairness clustering with outliers} problem~\cite{bera2019fair}: suppose each $P_j$ has a unique color, and we require that each color only takes $\frac{1}{m}$ of the total weight in each cluster $C_s$.
Though a number of algorithms have been proposed for fairness clustering, the study on the case with outliers is still quite limited, to the best of our knowledge.

\section{Our Clustering based LP  Algorithm}
\label{sec-alg}

In this section we propose a clustering based LP algorithm for solving the $k$-sparse WB with outliers problem, where our main idea is inspired by the observations of Claim~\ref{cl-m1} and  Claim~\ref{cl-m2}. Though the algorithm is simple, the analysis on the quality is the major challenge since the inliers and outliers are mixed without any prior knowledge. 

\paragraph{The clustering based LP algorithm.} 
Let $\mathbb{P}=\{P_1, P_2,$ $ \cdots, P_m\}$ be an instance of the $k$-sparse WB with outliers problem as Definition~\ref{def-kwboutlier}.   Assume $w_{\mathtt{min}}=\min_{1\leq j\leq m, p\in P_j}w_{P_j}(p)$, and denote by $\hat{z}=\lceil z/w_{\mathtt{min}}\rceil$ for convenience.  Our algorithm has the following three steps. 

\begin{enumerate}[label=\textbf{(\arabic*)}]
    \item We first run an $\alpha$-approximate $(k+\hat{z})$-means clustering algorithm $\mathcal{A}$ on each $P_j$ and obtain the  set $T_j$ of its $\lambda (k+\hat{z})$ cluster centers with some integer $\lambda\geq 1$ (as discussed in Remark~\ref{rem-result1}, we can run a bi-criteria approximation algorithm).
    \item Then, for each $T_j$ we consider the following fixed-support WB with outliers problem (as described in Remark~\ref{rem-fixoutlier}): the support of the barycenter is fixed to be the $O(k+\hat{z})$ points of $T_j$, and compute the optimal weight distribution over $T_j$ via LP.
    \item Let $\tilde{T}_1, \cdots, \tilde{T}_m$ be the obtained $m$ candidate WBs from Step (2). We return the best one, say $\tilde{T}_{j_0}$, which has the smallest cost over the $m$ candidates with respect to the cost (\ref{for-kwboutliers}). 
\end{enumerate}

The  theoretical quality guarantee of $\tilde{T}_{j_0}$ is given in Theorem~\ref{the-result1}. 
Since $|\mathtt{supp}(\tilde{T}_{j_0})|=O(k+\hat{z})$ that violates the $k$-sparse requirement, we can replace the algorithm $\mathcal{A}$ by a  $\beta$-approximate $k$-means clustering with $z$ outliers algorithm $\mathcal{B}$ in the above method. Then each $T_j$ should have the support size exactly equal to $k$ for $j=1,2, \cdots, m$. We still select the best candidate WB  $\tilde{T}_{j_0}$ by the same manner, and  return it as the solution for $k$-sparse WB with outliers. The improved result is shown in Theorem~\ref{the-result1-2}. 

\begin{theorem}
\label{the-result1}
Our clustering based LP Algorithm returns a solution $\tilde{T}_{j_0}$ for $k$-sparse WB with outliers and achieves the following quality guarantee:
\begin{eqnarray}
\mathtt{Cost}_{-z}(\mathbb{P}, \tilde{T}_{j_0})\leq (2+\sqrt{\alpha})^2\cdot\mathtt{Cost}_{-z}(\mathbb{P}, S_{\mathtt{opt}}).
\end{eqnarray}
\end{theorem}
We have multiple choices for the algorithm $\mathcal{A}$. For example, we can run the $(9+\epsilon)$-approximate local search algorithm~\cite{kanungo2002local} (but its running time is super linear since there are too many swap combinations that should be tested in the local search procedure). We also can run the $k$-means++ based algorithms~\cite{aggarwal2009adaptive,lattanzi2019better} to achieve an $O(1)$-approximation as discussed in Remark~\ref{rem-result1}. So the approximation factor in Theorem~\ref{the-result1} can be $O(1)$ as well. 

To prove Theorem~\ref{the-result1}, we need several key lemmas. Lemma~\ref{lem-1} shows that each obtained $T_j$ can approximately represent the corresponding $P_j$, even in the presence of \vspace{-0.25em} outliers. Lemma~\ref{lem-2} further shows that the set $T^{S_{\mathtt{opt}}}_j$, {\em i.e.,} the inliers of $T_j$ induced by $S_{\mathtt{opt}}$ (see Definition~\ref{def-fwd}), should yield an upper bound for the total cost where the bound is determined by the \vspace{-0.25em} distance between $P_j$ and $S_{\mathtt{opt}}$.  Also note that we cannot obtain \vspace{-0.25em} $T^{S_{\mathtt{opt}}}_j$ in reality since the optimal solution $S_{\mathtt{opt}}$ is always unknown to us; in fact we only use $T^{S_{\mathtt{opt}}}_j$ in our analysis \vspace{-0.25em} for bridging the gap between $\tilde{T}_{j_0}$ and $S_{\mathtt{opt}}$. Through Lemma~\ref{lem-2} we can prove that the selected best candidate $\tilde{T}_{j_0}$ yields the desired quality guarantee.

\begin{lemma}
\label{lem-1}
For each $1\leq j\leq m$, suppose the obtained cluster centers from $\mathcal{A}$ is $ T_j=\{t_1, t_2, \cdots, t_{\lambda(k+\hat{z})}\}$; also each weight $w_{T_j}(t_s)=$ the total weight of the $s$-th cluster. Then  $\mathcal{W}_{-z}(T_j, S_{\mathtt{opt}})\leq (1+\sqrt{\alpha}) \mathcal{W}_{-z}(P_j, S_{\mathtt{opt}})$.
\end{lemma}

\begin{proof} 
First, we consider the relationship between $k$-means clustering with $z$ outliers and $(k+\hat{z})$-means clustering. Intuitively, we can regard the result of $k$-means clustering with $z$ outliers as a special solution for the $(k+\hat{z})$-means clustering, where 
each outlier actually is a cluster of single point.  We then have the following claim (due to the space limit, the proof is shown in the supplement).

\begin{claim}
% \color{blue}
\label{cl-for-kzmeans}
% \begin{eqnarray}
$\mathtt{Mean}^{k+\hat{z}}(P_j) \leq \mathtt{Mean}^{k}_{-z}(P_j)$. 
% \end{eqnarray}
\end{claim}

From Claim \ref{cl-m1} we know that $\mathtt{Mean}^{k}_{-z}(P_j)\leq \mathcal{W}_{-z}^2(P_j, S_{\mathtt{opt}})$. Also, because $T_j$ is obtained from the $\alpha$-approximate algorithm $\mathcal{A}$, we have 
\begin{eqnarray}
\mathcal{W}(T_j, P_j)^2 &\leq& \alpha\mathtt{Mean}^{k+\hat{z}}(P_j) \nonumber\\
&\leq& \alpha\mathtt{Mean}^{k}_{-z}(P_j) \leq \alpha\mathcal{W}_{-z}^2(P_j, S_{\mathtt{opt}}). \label{for-lem-1-3}
\end{eqnarray}
According to Definition 1, we know that $\mathcal{W}_{-z}(T_j, S_{\mathtt{opt}}) = \mathcal{W}(T_j^{S_{\mathtt{opt}}}, S_{\mathtt{opt}})$, where $T_j^{S_{\mathtt{opt}}}$ is the inliers of $T_j$ induced by $S_{\mathtt{opt}}$. 
Note ${P_j^{S_{\mathtt{opt}}}}$ is a set with total weight $=n-z$, so we have $\mathcal{W}(T_j^{S_{\mathtt{opt}}}, S_{\mathtt{opt}})
\leq \mathcal{W}(T_j^{P_j^{S_{\mathtt{opt}}}}, S_{\mathtt{opt}})$. Thus, 
\vspace{-3pt}
\begin{eqnarray}
&& \mathcal{W}_{-z}(T_j, S_{\mathtt{opt}}) 
\leq \mathcal{W}(T_j^{P_j^{S_{\mathtt{opt}}}}, S_{\mathtt{opt}}).  \label{for-lem-1-4}
\end{eqnarray}
We also have the following bound 
\begin{eqnarray}
&& \mathcal{W}(T_j^{P_j^{S_{\mathtt{opt}}}}, S_{\mathtt{opt}}) \nonumber\\ % \text{\hspace{0.1in} (note $T_j^{P_j^{S_{\mathtt{opt}}}}$ is a set with total weight $=n-z$)} \nonumber\\
&\leq& \mathcal{W}(T_j^{P_j^{S_{\mathtt{opt}}}}, P_j^{S_{\mathtt{opt}}})+\mathcal{W}(P_j^{S_{\mathtt{opt}}}, S_{\mathtt{opt}}) \nonumber\\ 
& = & \mathcal{W}_{-z}(T_j, P_j^{S_{\mathtt{opt}}})+\mathcal{W}_{-z}(P_j, S_{\mathtt{opt}}) \nonumber\\
&\leq& \mathcal{W}(T_j, P_j)+\mathcal{W}_{-z}(P_j, S_{\mathtt{opt}}) \nonumber\\ 
&\leq& (1+\sqrt{\alpha}) \mathcal{W}_{-z}(P_j, S_{\mathtt{opt}}),  \label{for-lem-1-2}
\end{eqnarray}
where the first inequality follows from the triangle inequality of Wasserstein Distance, 
and the last inequality follows from  (\ref{for-lem-1-3}). Finally, we complete the proof by combining (\ref{for-lem-1-4}) and (\ref{for-lem-1-2}). 
\end{proof}

\begin{lemma}
\label{lem-2}
For any $1\leq j \leq m$, 
$\mathtt{Cost}_{-z}(\mathbb{P}, T^{S_{\mathtt{opt}}}_j)\leq(2+\sqrt{\alpha})\mathtt{Cost}_{-z}(\mathbb{P}, S_{\mathtt{opt}}) + (2+3\sqrt{\alpha}+\alpha)\mathcal{W}_{-z}^2(P_j, S_{\mathtt{opt}})$.
\end{lemma}
\begin{proof} 
For any $1\leq j_1\leq m$, we have
\begin{eqnarray}
&& \mathcal{W}_{-z}(P_{j_1}, T^{S_{\mathtt{opt}}}_j)
% &=\mathcal{W}(P^{T^{S_{\mathtt{opt}}}_j}_{j_1}, T^{S_{\mathtt{opt}}}_j)\\
\leq\mathcal{W}(P^{S_{\mathtt{opt}}}_{j_1}, T^{S_{\mathtt{opt}}}_j)\nonumber\\
&\leq& \mathcal{W}(P^{S_{\mathtt{opt}}}_{j_1}, S_{\mathtt{opt}})+\mathcal{W}(S_{\mathtt{opt}}, T^{S_{\mathtt{opt}}}_j) \nonumber\\
&=& \mathcal{W}_{-z}(P_{j_1}, S_{\mathtt{opt}})+\mathcal{W}_{-z}(T_j, S_{\mathtt{opt}})\nonumber\\
&\leq& \mathcal{W}_{-z}(P_{j_1}, S_{\mathtt{opt}})+(1+\sqrt{\alpha})\mathcal{W}_{-z}(P_j, S_{\mathtt{opt}}), \label{for-lem-2-1}
\end{eqnarray}
where the second inequality follows from the triangle inequality of Wasserstein distance and the third inequality follows from Lemma~\ref{lem-1}. Then we obtain the following bound by using (\ref{for-lem-2-1}): 
\begin{eqnarray}
&& \mathtt{Cost}_{-z}(\mathbb{P}, T^{S_{\mathtt{opt}}}_j) \nonumber\\
&\leq& \frac{1}{m}\sum\nolimits^m_{j_1=1}\Big(\mathcal{W}_{-z}(P_{j_1}, S_{\mathtt{opt}})\nonumber\\
&& +\;(1+\sqrt{\alpha})\mathcal{W}_{-z}(P_j, S_{\mathtt{opt}})\Big)^2 \nonumber\\
&\leq& \frac{1}{m}\sum\nolimits^m_{j_1=1}(2+\sqrt{\alpha})\mathcal{W}_{-z}^2(P_{j_1}, S_{\mathtt{opt}})\nonumber\\
&& +\;(2+3\sqrt{\alpha}+\alpha)\mathcal{W}_{-z}^2(P_j, S_{\mathtt{opt}}) \nonumber\\
&=& (2+\sqrt{\alpha})\mathtt{Cost}_{-z}(\mathbb{P}, S_{\mathtt{opt}}) \nonumber\\
&& +\; (2+3\sqrt{\alpha}+\alpha)\mathcal{W}_{-z}^2(P_j, S_{\mathtt{opt}}), \label{for-lem-2-2}
\end{eqnarray}
where the second inequality follows from the fact that $ (a+\delta b)^2\leq (1+\delta)a^2+(\delta^2+\delta)b^2$ for any numbers $a, b$, and $\delta$. 
\end{proof}

\begin{proof}

[\textbf{of Theorem~\ref{the-result1}}]
Because   $\tilde{T}_j$ is the optimal weight distribution over $T_j$,  we have $\mathtt{Cost}_{-z}(\mathbb{P}, \tilde{T}_j)\leq \mathtt{Cost}_{-z}(\mathbb{P}, T^{S_{\mathtt{opt}}}_j)$. For the best candidate $\tilde{T}_{j_0}$, we have
\begin{eqnarray}
&& \mathtt{Cost}_{-z}(\mathbb{P}, \tilde{T}_{j_0}) \leq \min\nolimits_{1\leq j\leq m}\mathtt{Cost}_{-z}(\mathbb{P}, T^{S_{\mathtt{opt}}}_j)\nonumber\\
&\leq& (2+\sqrt{\alpha})\;\mathtt{Cost}_{-z}(\mathbb{P}, S_{\mathtt{opt}}) \nonumber\\
&& +\;(2+3\sqrt{\alpha}+\alpha)\min\nolimits_{1\leq j\leq m}\mathcal{W}_{-z}^2(P_j, S_{\mathtt{opt}}) \label{for-the-result1-1}
\end{eqnarray}
based on Lemma~\ref{lem-2}. Also it is easy to know $\min_{1\leq j\leq m}\mathcal{W}_{-z}^2(P_j, S_{\mathtt{opt}})\leq \frac{1}{m}\sum^m_{j=1}\mathcal{W}_{-z}^2(P_j, S_{\mathtt{opt}}) = \mathtt{Cost}_{-z}(\mathbb{P}, S_{\mathtt{opt}})$, so (\ref{for-the-result1-1}) implies  $\mathtt{Cost}_{-z}(\mathbb{P}, \tilde{T}_{j_0})\leq  (2+\sqrt{\alpha})^2\mathtt{Cost}_{-z}(\mathbb{P}, S_{\mathtt{opt}}) $ which  completes the proof. 
\end{proof}

\begin{theorem}
\label{the-result1-2}
If we run the  $\beta$-approximate $k$-means clustering with $z$ outliers algorithm $\mathcal{B}$ instead of $\mathcal{A}$, and compute the optimal weight distribution with $2z$ outliers over $T_j$ in the clustering based LP algorithm, we have 
\begin{eqnarray}
\mathtt{Cost}_{-2z}(\mathbb{P}, \tilde{T}_{j_0})\leq (2+\sqrt{\beta})^2\cdot\mathtt{Cost}_{-z}(\mathbb{P}, S_{\mathtt{opt}}).
\end{eqnarray}
\end{theorem}
Comparing with Theorem~\ref{the-result1}, it is guaranteed that the output $\tilde{T}_{j_0}$ in Theorem~\ref{the-result1-2} is exactly $k$-sparse, with only a violation on the size of outliers. The total weight of discarded outliers is increased to $2z$; but usually $z$ is a value much smaller   than $n$  and thus we believe this influence is acceptable in practice.  For the algorithm $\mathcal{B}$, we can use the $O(1)$-approximate algorithms~\cite{friggstad2019approximation,krishnaswamy2018constant} (the algorithm of \cite{friggstad2019approximation} slightly violates the number of returned cluster centers to be $(1+\epsilon)k$ with an arbitrarily small value $\epsilon>0$). In practice, we can also use some faster algorithms like~\cite{chawla2013k,gupta2017local}
(though their theoretical guarantees are weaker than \cite{friggstad2019approximation,krishnaswamy2018constant}). 

% {\color{red}
The proof of Theorem~\ref{the-result1-2} is similar with that of Theorem~\ref{the-result1},  but the only major challenge is that we have to provide a   more complicated version for Lemma~\ref{lem-1} (which is Lemma~\ref{lem-3} below). For each $1\leq j\leq m$, let the obtained cluster centers  from $\mathcal{B}$ be $ T_j=\{t_1, t_2, \cdots, t_k\}$; the key difficult problem is that $T_j$ and $S_{\mathtt{opt}}$ may induce different inliers on $P_j$. To resolve this issue, we should make a deep analysis on the distribution of $T_j$ and prove the existence of a set $T'_j\preceq T_j$ who can play the same role as $T_j$ in Lemma~\ref{lem-1}. Due to the space limit, we leave the proofs of Lemma~\ref{lem-3} and Lemma~\ref{lem-4} to our supplement. 

\begin{lemma}
\label{lem-3}
There exists a weighted point set $T_j^{\prime}$ satisfying $\mathcal{W}_{-z}(S_{\mathtt{opt}}, T_j^{\prime})\leq (1+\sqrt{\beta}) \mathcal{W}_{-z}(P_j, S_{\mathtt{opt}})$, where  $\mathtt{supp}(T_j^{\prime})= \mathtt{supp}(T_j)$ and $w_{T_j^{\prime}}(T_j^{\prime}) = n - 2z$.
\end{lemma}

\begin{lemma}
\label{lem-4}
For any $1\leq j \leq m$, the weighted point set $T_j^{\prime}$ in Lemma~\ref{lem-3} also satisfies   
$\mathtt{Cost}_{-2z}(\mathbb{P}, T^{\prime}_j) \leq (2+\sqrt{\beta})\mathtt{Cost}_{-z}(\mathbb{P}, S_{\mathtt{opt}}) + (2+3\sqrt{\beta}+\beta) \mathcal{W}_{-z}^2(P_j, S_{\mathtt{opt}}) $.
\end{lemma}

\begin{proof}

[\textbf{of Theorem~\ref{the-result1-2}}]
Because of $\tilde{T}_j$ is the optimal weight distribution over $T_j$, so we have $\mathtt{Cost}_{-2z}(\mathbb{P}, \tilde{T}_j)\leq \mathtt{Cost}_{-2z}(\mathbb{P}, T^{\prime}_j)$. For the best candidate $\tilde{T}_{j_0}$, we have
\begin{eqnarray}
&&\mathtt{Cost}_{-2z}(\mathbb{P}, \tilde{T}_{j_0}) \leq \min\nolimits_{1\leq j\leq m}\mathtt{Cost}_{-2z}(\mathbb{P}, T^{\prime}_j) \nonumber\\
&\leq& (2+\sqrt{\beta})\mathtt{Cost}_{-z}(\mathbb{P}, S_{\mathtt{opt}}) \nonumber\\  
&&+ \;(2+3\sqrt{\beta}+\beta)\min\nolimits_{1\leq j\leq m}\mathcal{W}_{-z}^2(P_j, S_{\mathtt{opt}}) \nonumber\\ 
&\leq& (2+\sqrt{\beta})^2\mathtt{Cost}_{-z}(\mathbb{P}, S_{\mathtt{opt}}), 
\end{eqnarray}
where the second inequality follows from Lemme~\ref{lem-4} and the last inequality follows from $\min_{1\leq j\leq m}\mathcal{W}_{-z}^2(P_j, S_{\mathtt{opt}})\leq \mathtt{Cost}_{-z}(\mathbb{P}, S_{\mathtt{opt}})$.
\end{proof}

\textbf{Analysis on running time.} Let $\Gamma_1$ be the time complexity of $\mathcal{A}$ or $\mathcal{B}$ on each input set $P_j$, and let $\Gamma_2$ be the time complexity for solving the fixed-support WB for each $\tilde{T}_j$ as described in Remark~\ref{rem-fixoutlier}. Then the total time complexity of our algorithm is $m(\Gamma_1+\Gamma_2)$. Usually $\Gamma_1$ can be linear in $n$ if using the previous clustering algorithms (e.g.,
\cite{aggarwal2009adaptive})
%,lattanzi2019better,im2020fast,deshpande2020robust,grunau2022adapting}
; the complexity $\Gamma_2$ can be $\tilde{O}(mn^{7/3}\epsilon^{-4/3})$ if using the recent algorithm~\cite{lin2020fixed} ($\epsilon>0$ is a pre-specified small error).

\section{Improvement in Low-dimensional Space}
\label{sec-low}

We further consider a common case that the dimensionality $d$ is small ({\em e.g.,} the input $\mathbb{P}$ is a set of 2D images). In the previous section, we show that   the quality of our solution heavily depends on the clustering performance of the algorithm $\mathcal{A}$ or $\mathcal{B}$. So a natural question is  
\begin{center}
{\em Can we remove the dependence on the factors $\alpha$ and $\beta$ in our   quality guarantee?}
\end{center}
% {\em Can we remove the dependence on the factors $\alpha$ and $\beta$ in our quality guarantee?}
We answer this question in the affirmative. Our main idea is to generate the support through a  more sophisticated approach. Our algorithm contains the following two steps.

\textbf{(1) Generate the anchor points.} We still run the $\alpha$-approximate $(k+\hat{z})$-means clustering algorithm $\mathcal{A}$ on each $P_j$ and obtain the  set $T_j=\{t_1, \cdots, t_{\lambda(k+\hat{z})}\}$ of its cluster centers. 
To have a higher-quality fixed support for replacing   $T_j$, we utilize the low-dimensional coreset\footnote{Coreset is an algorithmic technique for representing large-scale data, which has been widely used for the optimization problems like clustering and regression~\cite{feldman2020introduction}. 
} technique~\cite{DBLP:conf/stoc/Har-PeledM04} to generate a set of ``anchor'' points $\hat{T}_j$ as follows. 
Denote by $\mathbb{B}(c, r)$ the ball centered at point $c$ with radius $r\geq 0$. We fix a $j$, and let $\bar{r}_j = \sqrt{\mathcal{W}(P_j, T_j)/n}$ and $\epsilon_1>0$. Then for $s=1, 2, \cdots, \lambda(k+\hat{z})$, we partition the ball $\mathbb{B}(t_s, n\bar{r}_j )$  into $\lceil \log n\rceil +1$ layers: $T_{j,s,0} = \mathbb{B}(t_s, \bar{r}_j)$ and $T_{j,s,h} = \mathbb{B}(t_s, \bar{r}_j2^h) \backslash \mathbb{B}(t_s, \bar{r}_j2^{h-1})$ for $h = 1, \cdots, \lceil\log(n)\rceil$. For each $h$-th layer, we can  build a grid with the side length $\bar{r}_j\epsilon_{1} 2^{h-1}/\sqrt{\alpha d}$; each point of $P_j\cap T_{j,s,h}$ is assigned to its nearest grid point, and each grid point has the weight equal to the total weight of the points assigned to it. Finally, we have the set $\hat{T}_j$ which contains all the weighted grid points of the  $\lceil \log n\rceil +1$ layers. 
The size $|\hat{T}_j| = O\big((k+\hat{z})\log(n)\alpha^{d/2}/\epsilon_{1}^d\big)$. We call the union set $\cup^m_{j=1}\hat{T}_j$ as the ``anchor points'' (which is actually the coreset in~\cite{DBLP:conf/stoc/Har-PeledM04}). 

 %Next, we construct some balls around each $t_i$ , and snap the points of $P_j$ to those balls. Let $T_{j,i,0} = \mathbb{B}(t_i, \bar{r})$ and $T_{j,i,h} = \mathbb{B}(t_i, \bar{r}2^h) \backslash \mathbb{B}(t_i, \bar{r}2^{h-1})$ for $h = 1, \cdots, \lceil\log(n)\rceil$. Partition $T_{j,i,h}$ into small ball with radius $r_{j,h} = \bar{r}_j\epsilon_{1} 2^{h-1}/\sqrt{\alpha}$ obtain $\hat{T}_{j,i,h}$, let $\hat{T}_j = \bigcup_{i,h}\hat{T}_{j,i,h}$. Note that $|\hat{T}_j| = O((k+\hat{z})\log(n)\lambda^d/\epsilon_{1}^d)$.

\textbf{(2) Construct the support.} Without loss of generality, we assume that the minimum and maximum pairwise distances of $\cup^m_{j=1}P_j$ are $1$ and $\Delta$, respectively.  For each anchor point $q\in \cup^m_{j=1}\hat{T}_j$, we draw $\lceil\log \Delta\rceil +1$ concentric balls $\mathbb{B}(q, 2^h)$, $h=0, 1, \cdots, \lceil\log \Delta\rceil$. Let $\epsilon_2>0$. Inside each ball $\mathbb{B}(q, 2^h)$, we build a grid of side length $\epsilon_{2} 2^{h-1}/\sqrt{d}$. We denote the union set of the $\lceil\log \Delta\rceil +1$ grids as $G_q$, and denote $\bar{G}=\bigcup_{q\in \bigcup^m_{j=1}\hat{T}_j}G_q$. 

Finally, we solve the fixed-support WB with outliers problem by using the LP method on $\bar{G}$ instead of the $T_j$s. The obtained solution is denoted by $\tilde{G}$. 

\begin{theorem}
\label{the-result2}
Our Algorithm returns a solution $\tilde{G}$  that has the following quality guarantee by setting $\epsilon_1=\epsilon_2=\epsilon/16$:
\begin{eqnarray}
\mathtt{Cost}_{-z}(\mathbb{P}, \tilde{G})\leq(1+\epsilon)\cdot\mathtt{Cost}_{-z}(\mathbb{P}, S_{\mathtt{opt}}).
\end{eqnarray}
\end{theorem}

Before proving Theorem~\ref{the-result2}, we provide the following two key lemmas first. Lemma~\ref{lem-5} shows that each obtained $\hat{T}_j$ can  efficiently preserve the Wasserstein distance error for $P_j$ within any arbitrarily small bound, 
 even in the presence of outliers. 
Lemma~\ref{lem-6} further shows that the instance $\mathbb{T}=\{\hat{T}_1, \hat{T}_2, \cdots, \hat{T}_m\}$ yields a barycenter on the fix support $\bar{G}$ which can approximately represent the barycenter on the input instance $\mathbb{P}$. 
The detailed proofs are placed to our supplement.

\begin{lemma}
\label{lem-5}
For each $1\leq j\leq m$, we have 
\begin{eqnarray}
\mathcal{W}(P_j, \hat{T}_j) \leq \sqrt{1.25}\epsilon_{1} \mathcal{W}_{-z}(P_j, S_{\mathtt{opt}}).
\end{eqnarray}
\end{lemma}

\begin{lemma}
\label{lem-6}
Let $\mathbb{T}=\{\hat{T}_1, \hat{T}_2, \cdots, \hat{T}_m\}$ be a new instance of $k$-sparse WB with outliers, then we have 
\begin{eqnarray}
\resizebox{.91\linewidth}{!}{$
\mathtt{Cost}_{-z}(\mathbb{T}, \tilde{G})\leq (1+\epsilon_{2})^2 (1+\sqrt{1.25}\epsilon_{1})^2 \mathtt{Cost}_{-z}(\mathbb{P}, S_{\mathtt{opt}})
$}.
\end{eqnarray}
\end{lemma}

\begin{proof}

[\textbf{of Theorem~\ref{the-result2}}]
We can combine Lemma~\ref{lem-5} and Lemma~\ref{lem-6} to complete the proof. Let $\delta>0$ be a parameter that will be determined later. First, we have 
\begin{eqnarray}
&&\mathtt{Cost}_{-z}(\mathbb{P}, \tilde{G})
= \frac{1}{m}\sum\nolimits_{j=1}^{m}\mathcal{W}_{-z}^2(P_j, \tilde{G}) \nonumber\\
&\leq& \frac{1}{m}\sum\nolimits_{j=1}^{m} \Big(\mathcal{W}(P_j, \hat{T}_j) + \mathcal{W}_{-z}(\hat{T}_j, \tilde{G})\Big)^2 \nonumber\\
&\leq& \frac{1}{m}\sum\nolimits_{j=1}^{m} \Big((1+\delta)\;\mathcal{W}^2(P_j, \hat{T}_j) \nonumber\\ 
&& +\; (1+\frac{1}{\delta})\;\mathcal{W}_{-z}^2(\hat{T}_j, \tilde{G})\Big), \label{for-the-result2-1}
\end{eqnarray}
where the second inequality follows from the generalized triangle inequality for any real numbers $\delta>0$, $a$, and $b$: $(a+b)^2\leq(1+\delta)a^2+(1+\frac{1}{\delta})b^2$. Then from (\ref{for-the-result2-1}) we have $\mathtt{Cost}_{-z}(\mathbb{P}, \tilde{G})\leq (1+\delta)\frac{1}{m}\sum_{j=1}^{m} \mathcal{W}_{-z}^2(P_j, S_{\mathtt{opt}}) + (1+\frac{1}{\delta})\mathtt{Cost}_{-z}(\mathbb{T}, \tilde{G}) \leq ((1+\delta)1.25\epsilon_{1}^2 \!+\! (1+\frac{1}{\delta}) (1+\epsilon_{2})^2 (1+\sqrt{1.25}\epsilon_{1})^2) \mathtt{Cost}_{-z}(\mathbb{P}, S_{\mathtt{opt}})$, 
where the second inequality follows from Lemma~\ref{lem-5} and Lemma~\ref{lem-6}. Finally we choose $\delta=1/(\sqrt{1.25}\epsilon_{1}(1+\epsilon_{2}) (1+\sqrt{1.25}\epsilon_{1}))$,  then we have $\mathtt{Cost}_{-z}(\mathbb{P}, \tilde{G})\leq (\sqrt{1.25}\epsilon_{1} + (1+\epsilon_{2})(1+\sqrt{1.25}\epsilon_{1}))^2\mathtt{Cost}_{-z}(\mathbb{P}, S_{\mathtt{opt}})$. By setting $\epsilon_1=\epsilon_2=\epsilon/16$ we obtain Theorem~\ref{the-result2}.
\end{proof}

\textbf{Running time analysis.} The running time is similar with the complexity of Section~\ref{sec-alg}, where the only difference is adding  the time complexity for building the anchor points and constructing the support in our algorithm. Note that the complexity of~\cite{DBLP:conf/stoc/Har-PeledM04} is linear in the input size $n$ for each $P_j$. The total complexity for this extra part is $\tilde{O}\Big(\log (\Delta )(k+\hat{z})\alpha^{d/2}\epsilon^{-2d} + n\Big)$.

\textbf{Extension in doubling metric.} Our result can be easily extended to the more general case in doubling metric. Informally speaking, the ``doubling dimension'' measures the intrinsic dimension of data (Euclidean dimension is one kind of special doubling dimension)~\cite{gupta2003bounded}. We show the extension with details in our supplement. 

\section{Experiments}
\label{sec-exp}
In this section, we illustrate the practical performance of our algorithms and study the significance of considering outliers for WB. Our experiments contain three parts. Firstly, we conduct the experiments on synthetic datasets, where the positions of the  barycenter supports are predefined, allowing us to compute the exact optimal objective value for measuring the approximation ratio of our algorithm. Secondly, 
we compare our algorithms with several baselines 
on real-world datasets. Finally, we provide the visualized results  on the MNIST dataset~\cite{lecun2010mnist}. 
Due to the space limit, the full experimental results are placed to our supplement. 

\textbf{Datasets.} In our synthetic datasets, we set the supports size $k\in[10,40]$ and the dimensionality $d\in[10,40]$; each instance comprises $m \in [2,10]$ different distributions, where each distribution consists of $n=20,000$ points.  The true barycenter supports are uniformly sampled within a hypercube with a side length of 10, and random weights are assigned to each center point. Points are randomly generated within Gaussian balls around the centers based on the assigned weights. We introduce outliers by uniformly sampling $z$ points for each distribution within the cube, with $z$ ranging from $0$ to $0.15\times n$. 

We also select three widely-used datasets from the UCI repository~\cite{UCI}: 
\textbf{Bank}~\cite{bank} ($4,521$ points in $\mathbb{R}^3$) represents the individual telephone calls during a marketing campaign, which
contains the information of the customers.
We have $m=3$  distributions categorized  based on marital status. 
\textbf{Credit card}~\cite{credit_card} ($30,000$ points in $\mathbb{R}^{14}$) includes the information about the credit card holders. We 
partitioned the data into $m=9$ distributions based on marriage and education. 
\textbf{Adult}~\cite{adult} ($32,561$ points in $\mathbb{R}^5$) represents the individual information from the 1994 U.S. Census. We 
partitioned it into $m=10$ distributions based on sex and race.  Finally, $5\%$ random noise are added to each dataset as outliers.

\textbf{Baselines and our implementation.} It is worth noting that there is no method that explicitly addresses $k$-sparse WB with outliers or fair clustering with outliers, to the best of our knowledge. We employ three baselines. First, following Remark~\ref{rem-fixoutlier}, we consider the fixed-support WB with outliers algorithm, utilizing $k$ random centers as support, and compute the optimal weight distribution via LP (denoted as ``Random\_$\mathcal{O}$'').
The other two baselines include a fair clustering algorithm that does not consider outliers (denoted as ``FC\_$\mathcal{O}$'')~\cite{bera2019fair}, and a non-fair clustering method considering outliers ``$k$-means-\,-\_$\mathcal{O}$''~\cite{chawla2013k}.
For the FC algorithm, we identify the farthest points in each class as outliers; for $k$-means-\,-, after obtaining the support positions, a new fair clustering solution can be obtained through LP. Additionally, we also test their three ``plain'' versions that do not discard  outliers, aiming to study the significance of considering outliers (denoted as ``Random'', ``FC'', ``$k$-means-\,-'', respectively).

In our implementation, we use the $k$-means++\cite{Arthur2007kmeansTA} as Algorithm $\mathcal{A}$ and $k$-means-\,-\cite{chawla2013k} as Algorithm $\mathcal{B}$; we also employ the LP solver~\cite{gurobi}  as the subroutine for solving fixed-support WB with outliers. To ensure a fair comparison, although Theorem~\ref{the-result1-2} suggests removing $2z$ outliers, we  remove only $z$ outliers in reality. Also, to keep $k$-sparsity for the result returned by $\mathcal{A}$, we only retain the top $k$ centers with the largest cluster sizes. We use ``Our\_$\mathcal{A}$'' and ``Our\_$\mathcal{B}$'' to denote them.

\vspace{0.05in}
\textbf{Results on synthetic datasets.} We compute the optimal cost $\mathtt{Cost}_{\mathtt{opt}}$ for the WB problem by using the pre-specified barycenter support. Subsequently, we execute our algorithm under various parameters to obtain the $\mathtt{Cost}$ and calculate the approximation ratio, defined as $\frac{\mathtt{Cost}}{\mathtt{Cost}_{\mathtt{opt}}}$. A subset of the results obtained by Algorithm $\mathcal{A}$  is presented in Table \ref{table:synthetic}, while the complete experimental outcomes are detailed in the appendix. We can see that our algorithm consistently achieves favorable approximation ratios across different dimensions and outlier proportions, where more than $70\%$ of them are less than $1.5$.

\textbf{Results on real datasets.} 
The results are illustrated in Figure \ref{fig:real}. 
As can be seen, even with only $5\%$ outliers, the plain versions of the three baselines take almost double costs than their counterparts who consider outliers. Moreover, our algorithms demonstrate even lower costs across all the datasets with different values of $k$. 
\vspace{-5pt}
\begin{table}[H]
  \centering
  \begin{tabular}{c@{\hspace{5pt}}c|@{\hspace{10pt}}c@{\hspace{5pt}}c@{\hspace{5pt}}c@{\hspace{5pt}}c@{\hspace{5pt}}c@{\hspace{5pt}}c@{\hspace{5pt}}c} %@{\hspace{5pt}}
  \toprule 
    & & \multicolumn{7}{c}{\textbf{Proportion of Outliers $z/n$}} \\
    $d$&$k$ & $0$   & $0.025$  & $0.05$  & $0.075$ & $0.1$ &$0.125$ &$0.15$  \\ 
    \cmidrule(r){1-9}
    \multirow{4}{*}{$10$}&10 & 1.321 & 1.380 & 1.477 & 1.651 & 1.547 & 1.452 & 1.493 \\
    &20 & 1.346 & 1.326 & 1.395 & 1.435 & 1.475 & 1.497 & 1.527 \\
    &30 & 1.370 & 1.375 & 1.397 & 1.434 & 1.476 & 1.496 & 1.558 \\
    &40 & 1.367 & 1.380 & 1.413 & 1.450 & 1.490 & 1.498 & 1.554 \\
    \cmidrule(r){1-9}
    \multirow{4}{*}{$20$}&10 & 1.332 & 1.412 & 1.695 & 1.714 & 1.746 & 1.353 & 1.399 \\
    &20 & 1.349 & 1.459 & 1.789 & 1.423 & 1.429 & 1.455 & 1.485 \\
    &30 & 1.373 & 1.468 & 1.412 & 1.441 & 1.485 & 1.497 & 1.538 \\
    &40 & 1.386 & 1.422 & 1.420 & 1.495 & 1.520 & 1.575 & 1.602 \\
\bottomrule
\end{tabular}
\caption{The approximation ratios of our algorithm for $m=10$.}
\label{table:synthetic}
\end{table}
\vspace{-15pt}
\begin{figure}[H]
  \centering
  \includegraphics[width=0.48\textwidth]{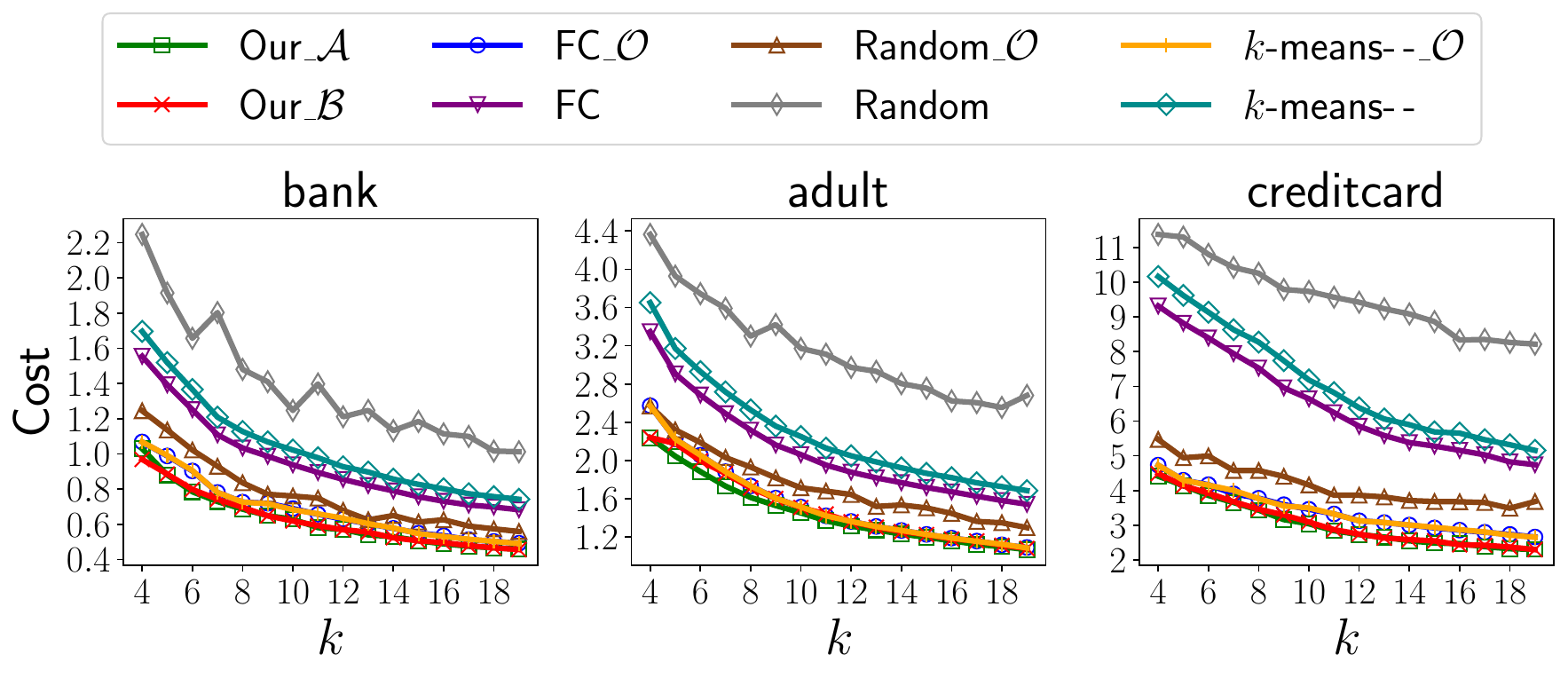}
  \caption{The obtained costs on real datasets.}
  \label{fig:real}
\end{figure}
\textbf{Visualized results.} 
In Figure \ref{fig:mnist}, we show the $40$-sparse barycenters obtained by Our\_$\mathcal{A}$ for  digit $0$-$4$ in the MNIST dataset, with $2\%$ of outliers removed from each digit. It is evident that the obtained set of $40$ points effectively captures the distinctive features for each digit.

\begin{figure}[H]
  \centering
  \includegraphics[width=0.45\textwidth]{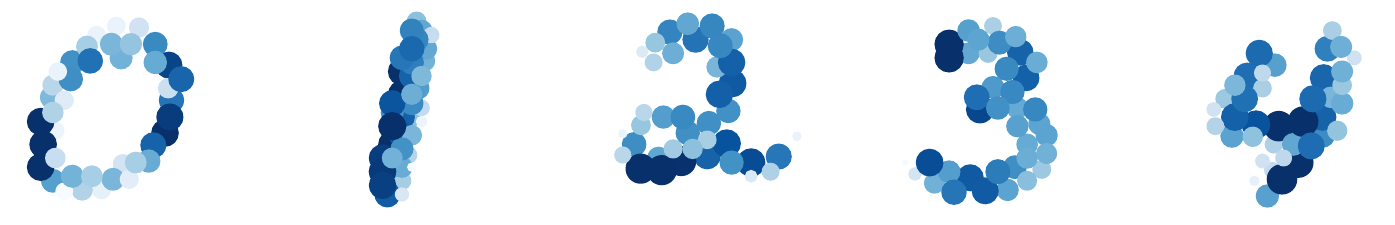}
  \caption{$k$-sparse WB obtained by Our\_$\mathcal{A}$ for $k=40$.}
  \label{fig:mnist}
\end{figure}

\section{Conclusions}
In this paper, we study the problem of $k$-sparse WB with outliers and present several efficient approximate algorithms with theoretical quality guarantees. Some omitted proofs are placed to our supplement. Following this work, there are several interesting problems deserved to study in future. 
For example, inspired by the local search method for designing the PTAS algorithm for ordinary $k$-means clustering with outliers~\cite{friggstad2019approximation}, 
an interesting theoretical question is that whether we can also apply it to achieve a PTAS for $k$-sparse WB with outliers in low-dimensional space. 

\newpage
\bibliographystyle{named}
\bibliography{ijcai24}

% \clearpage
\appendix

\section{Omitted Proofs for Theorem~\ref{the-result1-2}}
\subsection{The Proof of Lemma~\ref{lem-3}}
Before proving Lemma~\ref{lem-3}, we provide the following fact first. For any $w_P(P)=n, w_Q(Q)=n-z_1, Q^{\prime}\preceq Q, w_{Q^{\prime}}(Q^{\prime})=n-z_1-z_2 $, we have 
\begin{eqnarray}
&&\mathcal{W}_{-z_1-z_2}(P, Q^{\prime})\leq \mathcal{W}_{-z_2}(P^Q, Q^{\prime})\nonumber \\
&\leq& \mathcal{W}(P^Q, Q) = \mathcal{W}_{-z_1}(P, Q). \label{inequality-for-rwb-1}
\end{eqnarray}
\begin{proof}

[\textbf{of Lemma~\ref{lem-3}}]
Let $P_j^{T_j}$ and $P_j^{S_{\mathtt{opt}}}$ be the sets of inliers induced by $T_j$ and $S_{\mathtt{opt}}$, respectively. Note that $P_j^{T_j}$ and $P_j^{S_{\mathtt{opt}}}$ have the same support $\mathtt{supp}(P_j)$ but only with different weight distributions. %We define a new weighted set $P_j^{\prime}$ satisfies:
We define the ``$\cap$ ''operation between weighted. Let $P_j^{\prime} = P_j^{T_j} \cap P_j^{S_{\mathtt{opt}}}$ be the set with:
\begin{center}
{\em $\mathtt{supp}(P_j^{\prime})=\mathtt{supp}(P_j)$ ,  and for any $p\in \mathtt{supp}(P_j)$, 
$w_{P_j^{\prime}}(p)=\min\{w_{P_j^{T_j}}(p),w_{P_j^{S_{\mathtt{opt}}}}(p)\}$. }
\end{center}

Note that the total weight of $P_j^{\prime}$ is at least $n - 2z$. Let $T_j^{\prime}$ is any set satisfies $T_j^{\prime}\preceq T_j^{P_j^{\prime}}$ and $w_{T_j^{\prime}}(T_j^{\prime})=n-2z$, we have
\begin{eqnarray}
&&\mathcal{W}_{-z}(S_{\mathtt{opt}}, T_j^{\prime}) \leq \mathcal{W}(S^{P_j^{\prime}}_{\mathtt{opt}}, T_j^{P_j^{\prime}})  \nonumber\\
&\leq& \mathcal{W}(S^{P_j^{\prime}}_{\mathtt{opt}}, P_j^{\prime}) + \mathcal{W}(P_j^{\prime}, T_j^{P_j^{\prime}}) \nonumber\\
&=& \mathcal{W}_{-z}(S_{\mathtt{opt}}, P_j^{\prime}) + \mathcal{W}_{-z}(T_j, P_j^{\prime}) \nonumber\\
&\leq& \mathcal{W}(S_{\mathtt{opt}}, P_j^{S_{\mathtt{opt}}}) + \mathcal{W}(T_j, P_j^{T_j})\nonumber\\
&=& \mathcal{W}_{-z}(P_j, S_{\mathtt{opt}}) + \mathcal{W}_{-z}(P_j, T_j) \label{for-lem-3-1}
\end{eqnarray}
where the second inequality follows from (\ref{inequality-for-rwb-1}) and the fact that $P_j^{\prime}\preceq P_j^{S_{\mathtt{opt}}}$ and $P_j^{\prime}\preceq P_j^{T_j}$ .
Because $T_j$ is obtained from the algorithm $\mathcal{B}$, we have $\mathcal{W}_{-z}(P_j, T_j)^2 < \beta\mathtt{Cost}^{k}_{-z}(P_j)$. From Claim \ref{cl-m1} we know that $\mathtt{Cost}^{k}_{-z}(P)\leq \mathcal{W}_{-z}(P, S_{\mathtt{opt}})$, then (\ref{for-lem-3-1}) implies $\mathcal{W}_{-z}(S_{\mathtt{opt}}, T_j^{\prime})\leq (1+\sqrt{\beta}) \mathcal{W}_{-z}(P_j, S_{\mathtt{opt}})$.
\end{proof}

\subsection{The Proof of Lemma~\ref{lem-4}}
\begin{proof} 

%[\textbf{of Lemma~\ref{lem-4}}]
For any $1\leq j_1\leq m$, we have
\begin{eqnarray}
&&\mathcal{W}_{-2z}(P_{j_1}, T^{\prime}_j) \leq \mathcal{W}(P^{S_{\mathtt{opt}}^{T^{\prime}_j}}_{j_1}, T^{\prime}_j) \nonumber\\
&\leq& \mathcal{W}(P^{S_{\mathtt{opt}}^{T^{\prime}_j}}_{j_1}, S_{\mathtt{opt}}^{T^{\prime}_j})+\mathcal{W}(S_{\mathtt{opt}}^{T^{\prime}_j}, T^{\prime}_j) \nonumber \\
&=& \mathcal{W}_{-2z}(P_{j_1}, S_{\mathtt{opt}}^{T^{\prime}_j})+\mathcal{W}_{-z}(S_{\mathtt{opt}}, T^{\prime}_j) \nonumber\\
&\leq& \mathcal{W}_{-z}(P_{j_1}, S_{\mathtt{opt}})+(1+\sqrt{\beta})\mathcal{W}_{-z}(P_j, S_{\mathtt{opt}}) \label{for-lem-4-1}
\end{eqnarray}
where the second inequality follows from the triangle inequality of Wasserstein distance and the third inequality follows from Lemma~\ref{lem-3} and (\ref{inequality-for-rwb-1}). Then we obtain the following bound by applying (\ref{for-lem-4-1}):
\begin{eqnarray}
&&\mathtt{Cost}_{-2z}(\mathbb{P}, T^{\prime}_j)\nonumber \\
&=& \frac{1}{m}\sum^m_{j_1=1}\mathcal{W}_{-2z}^2(P_{j_1}, T^{\prime}_j) \nonumber\\
&\leq& \frac{1}{m}\sum^m_{j_1=1}(\mathcal{W}_{-z}(P_{j_1}, S_{\mathtt{opt}})\nonumber \\
&&+\;(1+\sqrt{\beta})\mathcal{W}_{-z}(P_j, S_{\mathtt{opt}}))^2 \nonumber\\
&\leq& \frac{1}{m}\sum^m_{j_1=1}(2+\sqrt{\beta})\mathcal{W}_{-z}^2(P_{j_1}, S_{\mathtt{opt}})\nonumber \\
&&+\;(2+3\sqrt{\beta}+\beta)\mathcal{W}_{-z}^2(P_j, S_{\mathtt{opt}}) \nonumber\\
&=& (2+\sqrt{\beta})\mathtt{Cost}_{-z}(\mathbb{P}, S_{\mathtt{opt}}) \nonumber \\
&&+\; (2+3\sqrt{\beta}+\beta)\mathcal{W}_{-z}^2(P_j, S_{\mathtt{opt}}). \label{for-lem-4-2}
\end{eqnarray}
where the third inequality follows from the fact that $ (a+\lambda b)^2\leq (1+\lambda)a^2+(\lambda^2+\lambda)b^2$ for any real numbers $a, b$, and $\lambda$.
\end{proof}

\section{Omitted Proofs for Theorem~\ref{the-result2}}
\subsection{The Proof of Lemma~\ref{lem-5}}
\begin{proof}

%[\textbf{of Lemma~\ref{lem-5}}]
Let $P_{j,t_s}$ denote the cluster induced by $t_s$ and $P_{j, t_s, h} = P_{j, t_s} \cap T_{j,s,h}$,$P_{j,*, h} = \bigcup_{t\in T} P_{j, t, h}$. We can rewrite the Wasserstein distance between $P_j$ and $T_j$ as
\begin{eqnarray}
&&\mathcal{W}(P_j, T_j)^2 \nonumber \\
&=& \sum_{t\in T_j}\sum_{h=0}^{\log n}\sum_{p\in P_{j,t,h}} w_{P_j}(p)||p-t||^2 \nonumber\\
&=& \sum_{t\in T}\sum_{h=1}^{\log n}\sum_{p\in P_{t,h}} w_{P_j}(p)||p-t||^2 \nonumber \\
&&+\; \sum_{t\in T}\sum_{p\in P_{t,0}} w_{P_j}(p)||p-t||^2 \nonumber\\
&\geq& \sum_{h=1}^{\log n}\bar{r}_j^{2}4^{h-1}\sum_{p\in P_{*, h}} w_{P_j}(p) \label{for-lem-5-1}
\end{eqnarray}
where the last inequality holds because the minimum radius of $T_{j,s,h}$ is $\bar{r}_j2^{h-1}$.
Let $\hat{t}_p$ denote the nearest point from $p$ in $\hat{T}$. Because   every point in $P_j$ is assigned to its nearest grid point, we can rewrite the Wasserstein distance between $P_j$ and $\hat{T}_j$ as
\begin{eqnarray}
&&\mathcal{W}(P_j, \hat{T}_j)^2 \nonumber \\
&=& \sum_{p\in P_{j,*,0}} w_{P_j}(p)||p-\hat{t}_p||^2 \nonumber \\
&&+\; \sum_{h=1}^{\log n}\sum_{p\in P_{j,*,h}} w_{P_j}(p)||p-\hat{t}_p||^2 \nonumber\\
&\leq& n  (\bar{r}_j\epsilon_{1}2^{-1}/ \sqrt{\alpha})^2 \nonumber \\
&&+\; \sum_{h=1}^{\log n}( \bar{r}_{j}\epsilon_{1}2^{h-1}/ \sqrt{\alpha})^2\sum_{p\in P_{j,*,h}} w_{P_j}(p) \nonumber\\
&=& \frac{\epsilon_{1}^2 }{4\alpha}\mathcal{W}(P_j, T_j)^2 \nonumber \\
&&+\; \epsilon_{1}^2 / \alpha \sum_{h=1}^{\log n}\bar{r}_j^{2}4^{h-1}\sum_{p\in P_{j,*, h}} w_{P_j}(p) \nonumber\\
&\leq& 1.25 \epsilon_{1}^2 / \alpha\mathcal{W}(P_j, T_j)^2 \nonumber \\
&\leq& 1.25 \epsilon_{1}^2 \mathcal{W}_{-z}(P_j, S_{\mathtt{opt}})^2 \label{for-lem-5-2}
\end{eqnarray}
where the second inequality holds because we build the grid with the side length $\bar{r}_j\epsilon_{1} 2^{h-1}/\sqrt{\alpha d}$ in each $h$-th layer and the last inequality follows from (\ref{for-lem-1-3}).
\end{proof}

\subsection{The Proof of Lemma~\ref{lem-6}}
\begin{proof}

%[\textbf{of Lemma~\ref{lem-6}}]
We use $S^{\mathbb{T}}_{\mathtt{opt}}$ to denote the optimal solution on $\mathbb{T}$. First, we prove that $\mathtt{Cost}_{-z}(\mathbb{T}, S^{\mathbb{T}}_{\mathtt{opt}}) \leq (1+\sqrt{1.25}\epsilon_{1}) \mathtt{Cost}_{-z}(\mathbb{P}, S_{\mathtt{opt}})$.

\begin{eqnarray}
&&\mathtt{Cost}_{-z}(\mathbb{T}, S^{\mathbb{T}}_{\mathtt{opt}}) \nonumber \\
&\leq& \mathtt{Cost}_{-z}(\mathbb{T}, S_{\mathtt{opt}}) = \frac{1}{m}\sum^m_{j=1}\mathcal{W}_{-z}^2(\hat{T}_j, S_{\mathtt{opt}}) \nonumber\\
&\leq& \frac{1}{m}\sum^m_{j=1}(\mathcal{W}(\hat{T}_j, P_j) + \mathcal{W}_{-z}(P_j, S_{\mathtt{opt}}))^2 \nonumber\\
&\leq& \frac{1}{m}\sum^m_{j=1}(\sqrt{1.25}\epsilon_{1} + 1)^2\mathcal{W}_{-z}^2(P_j, S_{\mathtt{opt}}) \nonumber\\
&=& (1+\sqrt{1.25}\epsilon_{1})^2 \mathtt{Cost}_{-z}(\mathbb{P}, S_{\mathtt{opt}}) \label{for-lem-6-1}
\end{eqnarray}
where the second inequality follows from the triangle inequality of Wasserstein distance and the third inequality follows from Lemma~\ref{lem-5}. It means that the instance $\mathbb{T}$ is a good approximation for the instance $\mathbb{P}$. Then we prove that there exists a weighted point set $\tilde{S}$ whose support is $\bar{G}$ which can well approximate $S^{\mathbb{T}}_{\mathtt{opt}}$.

For convenience, we define the distance function $\mathtt{d}(p, U):=\min_{q\in U}||p-q||$ for any point $p$ and point set $U$ in the space. 
It is easy to see that for any point $s$ satisfies $\mathtt{d}(s, \bigcup_{ \hat{T} \in \mathbb{T} }\hat{T})\leq\Delta$, we can conclude that $s$ lies within the ball $ \mathbb{B}(t_s, 2^{\lceil\log \mathtt{d}(s, \bigcup_{ \hat{T} \in \mathbb{T} }\hat{T})\rceil})$, where   $t_s = \arg\min\{||s-t||\mid t\in \bigcup_{ \hat{T} \in \mathbb{T} }\hat{T}\}$. So $s$ lies within the grid of side length $ \epsilon_{2} 2^{\lceil\log \mathtt{d}(s, \bigcup_{ \hat{T} \in \mathbb{T} }\hat{T})\rceil-1}/\sqrt{d}\leq \epsilon_{2} \mathtt{d}(s, \bigcup_{ \hat{T} \in \mathbb{T} }\hat{T})/\sqrt{d} $. Then we have $\mathtt{d}(s, \bar{G})\leq \epsilon_{2} \mathtt{d}(s, \bigcup_{ \hat{T} \in \mathbb{T} }\hat{T})$ . We suppose $S^{\mathbb{T}}_{\mathtt{opt}}=\{s_1, \cdots, s_k\}$, let $\mathtt{supp}(\tilde{S})=\{ \tilde{s}_1, \cdots,\tilde{s}_k \}$ where $\tilde{s}_i$ is the nearest point in $\bar{G}$ to $s_i$ and $w_{\tilde{S}}(\tilde{s}_i) = w_{S^{\mathbb{T}}_{\mathtt{opt}}}(s_i)$. Then we have $||s_i-\tilde{s}_i||=\mathtt{d}(s_i,\bar{G})\leq \epsilon_{2} \mathtt{d}(s_i, \bigcup_{ \hat{T} \in \mathbb{T} }\hat{T})$. From the above we can conclude that
\begin{eqnarray}
&&\mathcal{W}(\tilde{S}, S^{\mathbb{T}}_{\mathtt{opt}})^2 \nonumber \\
&=& \sum_{s\in S^{\mathbb{T}}_{\mathtt{opt}}} w_{S^{\mathbb{T}}_{\mathtt{opt}}}(s)||s-\tilde{s}||^2 \nonumber \\
&\leq& \epsilon_{2}^2 \sum_{s\in S^{\mathbb{T}}_{\mathtt{opt}}} w_{S^{\mathbb{T}}_{\mathtt{opt}}}(s) \mathtt{d}(s,\cup_{ \hat{T} \in \mathbb{T} }\hat{T})^2\nonumber \\
&\leq&  \epsilon_{2}^2 \mathcal{W}_{-z}(\hat{T}_j, S^{\mathbb{T}}_{\mathtt{opt}})^2 \label{for-lem-6-2}
\end{eqnarray}
holds for any $j$.
%which bound the distance between $\tilde{S}$ and $S^{\mathbb{T}}_{\mathtt{opt}}$. 
So we have
\begin{eqnarray}
&&\mathtt{Cost}_{-z}(\mathbb{T}, \tilde{S}) = \frac{1}{m}\sum_{j=1}^{m}\mathcal{W}_{-z}^2(\hat{T}_j, \tilde{S})\nonumber \\
&\leq& \frac{1}{m}\sum_{j=1}^{m}(\mathcal{W}_{-z}(\hat{T}_j, S^{\mathbb{T}}_{\mathtt{opt}}) + \mathcal{W}(S^{\mathbb{T}}_{\mathtt{opt}}, \tilde{S}))^2 \nonumber\\
&\leq& (1+\epsilon_{2})^2 \mathtt{Cost}_{-z}(\mathbb{T}, S^{\mathbb{T}}_{\mathtt{opt}}) \nonumber\\
&\leq& (1+\epsilon_{2})^2 (1+\sqrt{1.25}\epsilon_{1})^2 \mathtt{Cost}_{-z}(\mathbb{P}, S_{\mathtt{opt}}) \label{for-lem-6-3}
\end{eqnarray}
where the first inequality follows from the triangle inequality of Wasserstein distance and the second inequality follows from (\ref{for-lem-6-3}). It is easy to know $\mathtt{Cost}_{-z}(\mathbb{T}, \tilde{G})\leq \mathtt{Cost}_{-z}(\mathbb{T}, \tilde{S}) \leq (1+\epsilon_{2})^2 (1+\sqrt{1.25}\epsilon_{1})^2 \mathtt{Cost}_{-z}(\mathbb{P}, S_{\mathtt{opt}})$, so the proof is completed. 
\end{proof}

\subsection{Extension in Doubling Metric}
Doubling dimension is a concept used in computational geometry and algorithmic analysis to measure the intrinsic complexity of metric spaces. It provides a quantitative measure of how fast the size of a metric space increases as its radius doubles.
\begin{definition}[Doubling Dimension]
\label{def-doubling-dimension}
The doubling dimension of a data set $P$ is the smallest number $\rho>0$, such that for any $p \in P$ and $r \geq 0, P \cap \mathbb{B}(p, 2 r)$ is always covered by the union of at most $2^\rho$ balls with radius $r$.
\end{definition}

By recursively applying Definition \ref{def-doubling-dimension} $\log(1/\epsilon)$ times, we can conclude that for a given dataset $P$ within the ball $\mathbb{B}(p, r)$, if $P$ has a doubling dimension of $\rho$, then it can be covered by $(1/\epsilon)^\rho$ balls with a radius of $\epsilon r$. 

In the algorithm proposed in Section~\ref{sec-low}, we use a grid-based approach in the ``Generate the anchor points'' and ``Construct the support'' steps to find a set of points within the ball $\mathbb{B}(p,r2^h)$ such that their distances to other points in the ball are less than $\epsilon r2^{h-1}$. If the dataset has a low doubling dimension instead of Euclidean dimension, we can use the $2$-approximation Gonzalez's algorithm for the $k$-center problem~\cite{gonzalez1985clustering}. Because the points in $\mathbb{B}(p,r2^h)$ can be covered by $(4/\epsilon)^\rho$ balls of radius $\epsilon r2^{h-2}$,  with a $2$-approximation $k$-center clustering algorithm, we can obtain $(4/\epsilon)^\rho$ points such that their distances to other points in the ball are less than $\epsilon r2^{h-1}$. Namely, by using the Gonzalez's algorithm  we can achieve the similar result as the grid-based approach. Then the remaining parts of the proof can be extended to the case of low doubling dimension.

\section{Results for Two-side version}

First we give a more general definition of robust Wasserstein distance that both $P$ and $Q$ may contain outliers.

\begin{definition}[Wasserstein distance with $(z_1,z_2)$ outliers]
\label{def-2sidewd}
Let $n>z_1,z_2\geq 0$. Suppose $P=\{p_1, p_2, \cdots, p_{n_1}\}$ and $Q=\{q_1, q_2, \cdots, q_{n_2}\}$ are two sets of  weighted points in   $\mathbb{R}^d$; also $\sum^{n_1}_{i=1}w_P(p_i)=n $ and $ \sum^{n_2}_{j=1}w_Q(q_j)=n-z_1+z_2$.  
The \textbf{Wasserstein distance with $(z_1,z_2)$ outliers} from $P$ to $Q$ is 
\begin{eqnarray}
\mathcal{W}_{-z_1,-z_2}(P, Q)=\min_{P^{\prime}\in \mathbb{P}^P_{-z_1}, Q^{\prime}\in \mathbb{P}^Q_{-z_2}} \mathcal{W}(P^{\prime},Q^{\prime}).\label{for-2sidewd}
\end{eqnarray} 
The sets $P^{Q,z_1}$ and $Q^{P,z_2}$ denote the corresponding optimal subsets from $\mathbb{P}^P_{-z_1}$ and $\mathbb{P}^Q_{-z_2}$,  
 %$$=\arg\min_{P^{\prime}\in \mathbb{P}^P_{-z_1}, Q^{\prime}\in \mathbb{P}^Q_{-z_2}} \mathcal{W}(P^{\prime}, Q^{\prime})$ is
 which are called ``\textbf{the inliers of $P$ induced by $Q$ and $z_1$}'' and ``\textbf{the inliers of $Q$ induced by $P$ and $z_2$}'', respectively. Clearly, the definition of $P^Q$  in definition~\ref{def-fwd} is equivalent to $P^{Q,z_1}$ defined here when $w_Q(Q)=w_P(P)-z_1$.
\end{definition}

Similarly, we can define the two-side version of the $k$-sparse Wasserstein barycenter.

\begin{definition}[$k$-sparse WB with $(z_1,z_2)$ outliers]
\label{def-kwboutlier-2side}
Let $n>z_1,z_2\geq 0$. Suppose the input $\mathbb{P}$ contains $m\geq 1$ weighted point sets $P_1, P_2, \cdots, P_m$ where each $P_j$ has the total weight $n$, the problem of $k$-sparse Wasserstein Barycenter with $(z_1,z_2)$ outliers for a given $k\in \mathbb{Z}^+$ is to construct a new weighted point set $S$ with  $|\mathtt{supp}(S)|=k$, such that 
\begin{eqnarray}
\mathtt{Cost}_{-z_1,-z_2}(\mathbb{P}, S)=\frac{1}{m}\sum^m_{j=1}\mathcal{W}_{-z_1,-z_2}^2(P_j, S) \label{for-kwboutliers-2side}
\end{eqnarray}
is minimized. 
\end{definition}
Throughout this section, we always use $S_{\mathtt{opt}}$ to denote the optimal solution of (\ref{for-kwboutliers-2side}). With the above definitions in place, we can obtain the more general versions of Theorem \ref{the-result1} and Theorem \ref{the-result1-2}. If we set $z_2=0$,  Theorem \ref{the-result1-2side} and Theorem \ref{the-result1-2-2side} are equivalent to Theorem \ref{the-result1} and Theorem \ref{the-result1-2}.

\begin{theorem}
\label{the-result1-2side}
% Our clustering based LP Algorithm returns a solution $\tilde{T}_{j_0}$ for $k$-sparse WB with outliers and achieves the following quality guarantee:
If we compute the optimal weight distribution with $((z_1+z_2), z_2)$ outliers over $T_j$ in the clustering based LP algorithm, we have 
\begin{eqnarray}
&&\mathtt{Cost}_{-(z_1+z_2),-z_2}(\mathbb{P}, \tilde{T}_{j_0})\nonumber \\
&\leq& (2+\sqrt{\alpha})^2\cdot\mathtt{Cost}_{-z_1,-z_2}(\mathbb{P}, S_{\mathtt{opt}}).
\end{eqnarray}
\end{theorem}

\begin{theorem}
\label{the-result1-2-2side}
If we run the  $\beta$-approximate $k$-means clustering with $z_1$ outliers algorithm $\mathcal{B}$ instead of $\mathcal{A}$, and compute the optimal weight distribution with $((2z_1+z_2), z_2)$ outliers over $T_j$ in the clustering based LP algorithm, we have 
\begin{eqnarray}
&&\mathtt{Cost}_{-(2z_1+z_2), -z_2}(\mathbb{P}, \tilde{T}_{j_0})\nonumber \\
&\leq& (2+\sqrt{\beta})^2\cdot\mathtt{Cost}_{-z_1,-z_2}(\mathbb{P}, S_{\mathtt{opt}}).
\end{eqnarray}
\end{theorem}

\subsection{The Proof of Theorem~\ref{the-result1-2side}}

For the convenience of later proofs, we present some variations of the triangle inequality under the robust Wasserstein distance: For any weighted set $O,P,Q$ with $w_{P}(P)=n, w_{O}(O)=n-z_1, w_{Q}(Q)=n-z_1+z_2$, we have
\begin{eqnarray}
&&\mathcal{W}_{-z_1,-z_2}(P, Q)  \nonumber \\
&=&  \mathcal{W}(P^{Q,z_1}, Q^{P,z_2}) \leq \mathcal{W}(P^{O,z_1}, Q^{O,z_2})   \nonumber \\
&\leq& \mathcal{W}(P^{O,z_1}, O) + \mathcal{W}(Q^{O,z_2}, O) \nonumber \\
&=& \mathcal{W}_{-z_1}(P, O) + \mathcal{W}_{-z_2}(Q, O) 
\label{triangle-inequality-for-rwb-1}
\end{eqnarray}
For any weighted set $O,P,Q$ with $w_{P}(P)=n, w_{O}(O)=n-z_1, w_{Q}(Q)=n-(z_1+z_2)$, we have
\begin{eqnarray}
&&\mathcal{W}_{-(z_1+z_2)}(P, Q)  \nonumber \\
&=& \mathcal{W}(P^{Q}, Q) \leq \mathcal{W}(P^{O^Q}, Q)   \nonumber \\
&\leq& \mathcal{W}(P^{O^Q}, O^Q) + \mathcal{W}(O^Q, O)  \nonumber \\
& = & \mathcal{W}_{-(z_1+z_2)}(P, O^Q) + \mathcal{W}_{-z_2}(O, Q) \nonumber \\
&\leq& \mathcal{W}_{-z_1}(P, O) + \mathcal{W}_{-z_2}(O, Q) 
\label{triangle-inequality-for-rwb-2}
\end{eqnarray}
% \begin{eqnarray}
% \mathcal{W}_{-(z_1+z_2)}(P, Q) 
% & = & \mathcal{W}(P^{Q}, Q) \leq \mathcal{W}(P^{O^Q}, Q) \nonumber \\
% &\leq& \mathcal{W}(P^{O^Q}, O^Q) + \mathcal{W}(O^Q, O)  \nonumber \\
% & = & \mathcal{W}_{-(z_1+z_2)}(P, O^Q) + \mathcal{W}_{-z_2}(O, Q) \nonumber \\
% &\leq& \mathcal{W}_{-z_1}(P, O) + \mathcal{W}_{-z_2}(O, Q) 
% \label{triangle-inequality-for-rwb-2}
% \end{eqnarray}
where the last inequality follows from (\ref{inequality-for-rwb-1}).

\begin{lemma}
\label{lem-7}
For each $1\leq j\leq m$, let the obtained cluster centers $ T_j=\{t_1, t_2, \cdots, t_{\lambda(k+\hat{z})}\}$ from $\mathcal{A}$; also each weight $w_{T_j}(t_s)=$ the total weight of the $s$-th cluster. Then  $\mathcal{W}_{-z_1,-z_2}(T_j, S_{\mathtt{opt}})\leq (1+\sqrt{\alpha}) \mathcal{W}_{-z_1,-z_2}(P_j, S_{\mathtt{opt}})$.
\end{lemma}

\begin{proof}
Similar with (\ref{for-lem-1-3}) we have 
\begin{eqnarray}
&&\mathcal{W}(T_j, P_j)^2 \leq \alpha\mathtt{Mean}^{k+\hat{z_1}}(P_j) \nonumber\\ 
&\leq& \alpha\mathtt{Mean}^{k}_{-z_1}(P_j) \leq \alpha\mathcal{W}_{-z_1,-z_2}^2(P_j, S_{\mathtt{opt}}). \label{for-lem-7-1}
\end{eqnarray}
Then we have 
\begin{eqnarray}
&&\mathcal{W}_{-z_1,-z_2}(T_j, S_{\mathtt{opt}}) \nonumber\\
% &\leq& \mathcal{W}(T_j^{P_j^{S_{\mathtt{opt}}}}, S_{\mathtt{opt}})  \nonumber\\
&\leq& \mathcal{W}_{-z_1}(T_j, P_j^{S_{\mathtt{opt}},z_1})+\mathcal{W}_{-z_2}(S_{\mathtt{opt}}, P_j^{S_{\mathtt{opt}},z_1}) \nonumber\\
& \leq & \mathcal{W}(T_j, P_j)+ \mathcal{W}_{-z_2}(S_{\mathtt{opt}}, P_j^{S_{\mathtt{opt}},z_1}) \nonumber\\
& = & \mathcal{W}(T_j, P_j)+ \mathcal{W}_{-z_1,-z_2}(P_j, S_{\mathtt{opt}}) \nonumber\\
&\leq& (1+\sqrt{\alpha})\mathcal{W}_{-z_1,-z_2}(P_j, S_{\mathtt{opt}})  \label{for-lem-7-2}
\end{eqnarray}
where the first inequality follows from (\ref{triangle-inequality-for-rwb-1}) and the last inequality follows from (\ref{for-lem-7-1}).
\end{proof}

\begin{lemma}
\label{lem-8}
For any $1\leq j \leq m$, 
\begin{eqnarray}
&&\mathtt{Cost}_{-(z_1+z_2),-z_2}(\mathbb{P}, T_j^{S_{\mathtt{opt}},z_1}) \nonumber\\
&\leq& (2+\sqrt{\alpha})\mathtt{Cost}_{-z_1,-z_2}(\mathbb{P}, S_{\mathtt{opt}}) \nonumber\\
&& +\; (2+3\sqrt{\alpha}+\alpha)\mathcal{W}_{-z_1,-z_2}^2(P_j, S_{\mathtt{opt}})\nonumber
\end{eqnarray}
% $$\mathtt{Cost}_{-(z_1+z_2),-z_2}(\mathbb{P}, T_j^{S_{\mathtt{opt}},z_1}) \leq (2+\sqrt{\alpha})\mathtt{Cost}_{-z_1,-z_2}(\mathbb{P}, S_{\mathtt{opt}}) + (2+3\sqrt{\alpha}+\alpha)\mathcal{W}_{-z_1,-z_2}^2(P_j, S_{\mathtt{opt}})$$.
\end{lemma}

\begin{proof}
For any $1\leq j_1\leq m$, Let $S_{\mathtt{opt}}^{j_1\prime}$ be an arbitrary weighted point set that satisfies  $S_{\mathtt{opt}}^{j_1\prime} \preceq S_{\mathtt{opt}}^{T_j,z_2}\cap S_{\mathtt{opt}}^{P_{j_1},z_2}$ and $w_{S_{\mathtt{opt}}^{j_1\prime}}(S_{\mathtt{opt}}^{j_1\prime}) = n - z_1 - z_2$, we have
\begin{eqnarray}
&&\mathcal{W}_{-(z_1+z_2),-z_2}(P_{j_1}, T_j^{S_{\mathtt{opt}},z_1}) \nonumber\\
% &=\mathcal{W}(P^{T^{S_{\mathtt{opt}}}_j}_{j_1}, T^{S_{\mathtt{opt}}}_j)\\
&\leq& \mathcal{W}_{-(z_1+z_2)}(P_{j_1}, S_{\mathtt{opt}}^{j_1\prime}) + \mathcal{W}_{-z_2}(T_j^{S_{\mathtt{opt}},z_1}, S_{\mathtt{opt}}^{j_1\prime}) \nonumber\\
&\leq& \mathcal{W}_{-z_1}(P_{j_1}, S_{\mathtt{opt}}^{P_{j_1},z_2}) + \mathcal{W}(T_j^{S_{\mathtt{opt}},z_1}, S_{\mathtt{opt}}^{T_j,z_2}) \nonumber\\
& = & \mathcal{W}_{-z_1,-z_2}(P_{j_1}, S_{\mathtt{opt}}) + \mathcal{W}_{-z_1,-z_2}(T_j, S_{\mathtt{opt}}) \nonumber\\
&\leq& \mathcal{W}_{-z_1,-z_2}(P_{j_1}, S_{\mathtt{opt}}) \nonumber\\
&& +\  (1+\sqrt{\alpha})\mathcal{W}_{-z_1,-z_2}(P_j, S_{\mathtt{opt}}) \label{for-lem-8-1}
\end{eqnarray}
where the first inequality follows from (\ref{triangle-inequality-for-rwb-1}), the second inequality follows from (\ref{inequality-for-rwb-1}) and the last inequality follows from Lemma~\ref{lem-7}. Then we obtain the following bound by applying (\ref{for-lem-8-1}):
\begin{eqnarray}
&&\mathtt{Cost}_{-(z_1+z_2),-z_2}(\mathbb{P}, T_j^{S_{\mathtt{opt}},z_1})\nonumber\\
&=& \frac{1}{m}\sum^m_{j_1=1}\mathcal{W}_{-2z}^2(P_{j_1}, T^{\prime}_j) \nonumber\\
&\leq& \frac{1}{m}\sum^m_{j_1=1}(\mathcal{W}_{-z_1,-z_2}(P_{j_1}, S_{\mathtt{opt}}) \nonumber\\
&& +\; (1+\sqrt{\alpha})\mathcal{W}_{-z_1,-z_2}(P_j, S_{\mathtt{opt}}))^2 \nonumber\\
&\leq& \frac{1}{m}\sum^m_{j_1=1}(2+\sqrt{\alpha})\mathcal{W}_{-z_1,-z_2}^2(P_{j_1}, S_{\mathtt{opt}})\nonumber\\
&& +\;(2+3\sqrt{\alpha}+\alpha)\mathcal{W}_{-z_1,-z_2}^2(P_j, S_{\mathtt{opt}}) \nonumber\\
&=& (2+\sqrt{\alpha})\mathtt{Cost}_{-z_1,-z_2}(\mathbb{P}, S_{\mathtt{opt}}) \nonumber\\
&& +\; (2+3\sqrt{\alpha}+\alpha)\mathcal{W}_{-z_1,-z_2}^2(P_j, S_{\mathtt{opt}}) \label{for-lem-8-2}
\end{eqnarray}
where the third inequality follows from the fact that $ (a+\lambda b)^2\leq (1+\lambda)a^2+(\lambda^2+\lambda)b^2$ for any real numbers $a, b$, and $\lambda$.
\end{proof}

\begin{proof} 

[\textbf{of Theorem~\ref{the-result1-2side}}] 
Because of $\tilde{T}_j$ is the optimal weight distribution over $T_j$, so we have $\mathtt{Cost}_{-(z_1+z_2),-z_2}(\mathbb{P}, \tilde{T}_{j})\leq \mathtt{Cost}_{-(z_1+z_2),-z_2}(\mathbb{P}, T_j^{S_{\mathtt{opt}},z_1})$. For the best candidate $\tilde{T}_{j_0}$, we have
\begin{eqnarray}
&&\mathtt{Cost}_{-(z_1+z_2),-z_2}(\mathbb{P}, \tilde{T}_{j_0}) \nonumber\\
&\leq& \min_{1\leq j\leq m} \mathtt{Cost}_{-(z_1+z_2),-z_2}(\mathbb{P}, T_j^{S_{\mathtt{opt}},z_1}) \nonumber\\
&\leq& (2+\sqrt{\alpha})\mathtt{Cost}_{-z_1,-z_2}(\mathbb{P}, S_{\mathtt{opt}}) \nonumber\\
&& +\; (2+3\sqrt{\alpha}+\alpha)\min_{1\leq j\leq m}\mathcal{W}_{-z_1,-z_2}^2(P_j, S_{\mathtt{opt}})  \nonumber\\ 
&\leq& (2+\sqrt{\alpha})^2\cdot\mathtt{Cost}_{-z_1,-z_2}(\mathbb{P}, S_{\mathtt{opt}}) , 
\end{eqnarray}
where the second inequality follows from Lemme~\ref{lem-8} and the last inequality follows from $\min_{1\leq j\leq m}\mathcal{W}_{-z_1,-z_2}^2(P_j, S_{\mathtt{opt}}) \leq \mathtt{Cost}_{-z_1,-z_2}(\mathbb{P}, S_{\mathtt{opt}})$.
\end{proof}

\subsection{The Proof of Theorem~\ref{the-result1-2-2side}}

\begin{lemma}
\label{lem-9}
% There exists a weighted point set $T_j^{\prime}$ with $\mathtt{supp}(T_j^{\prime})= \mathtt{supp}(T_j)$ and $w_{T_j^{\prime}}(T_j^{\prime}) = n - 2z$ that satisfies $\mathcal{W}_{-z}(S_{\mathtt{opt}}, T_j^{\prime})\leq (1+\sqrt{\beta}) \mathcal{W}_{-z}(P_j, S_{\mathtt{opt}})$.
There exists a weighted point set $T_j^{\prime}$ with $\mathtt{supp}(T_j^{\prime})= \mathtt{supp}(T_j)$ and $w_{T_j^{\prime}}(T_j^{\prime}) = n - 2z_1$ that satisfies $\mathcal{W}_{-(z_1+z_2)}(S_{\mathtt{opt}}, T_j^{\prime}) \leq (1+\sqrt{\beta}) \mathcal{W}_{-z_1,-z_2}(P_j, S_{\mathtt{opt}})$.
\end{lemma}
\begin{proof}
Let $P_j^{\prime} = P_j^{T_j,z_1} \cap P_j^{S_{\mathtt{opt}},z_1}$.  
Note that the total weight of $P_j^{\prime}$ is at least $n - 2z_1$, assume that $w_{P_j^{\prime}}(P_j^{\prime}) = n-z_1-z^{\prime}$, $0\leq z^{\prime}\leq z_1$. Let $T_j^{\prime}$ is any set satisfies $T_j^{\prime}\preceq T_j^{P_j^{\prime}}$ and $w_{T_j^{\prime}}(T_j^{\prime})=n-2z_1$, we have
\begin{eqnarray}
&& \mathcal{W}_{-(z_1+z_2)}(S_{\mathtt{opt}}, T_j^{\prime})  \nonumber\\
&\leq&  \mathcal{W}_{-(z_2+z^{\prime})}(S_{\mathtt{opt}}, P_j^{\prime})  + \mathcal{W}_{-(z_1-z^{\prime})}(P_j^{\prime}, T_j^{\prime}) \nonumber\\
&\leq& \mathcal{W}_{-z_2}(S_{\mathtt{opt}}, P_j^{S_{\mathtt{opt}},z_1})  + \mathcal{W}_{-z_1}(P_j^{T_j,z_1}, T_j^{\prime}) \nonumber\\
&\leq& \mathcal{W}_{-z_2}(S_{\mathtt{opt}}, P_j^{S_{\mathtt{opt}},z_1})  + \mathcal{W}(P_j^{T_j,z_1}, T_j) \nonumber\\
&=& \mathcal{W}_{-z_1,-z_2}(P_j, S_{\mathtt{opt}}) + \mathcal{W}_{-z_1}(P_j, T_j) \label{for-lem-9-1}
\end{eqnarray}
where the first inequality follows from (\ref{triangle-inequality-for-rwb-2}), the second and third inequality follows from (\ref{inequality-for-rwb-1}) and the fact that $P_j^{\prime}\preceq P_j^{S_{\mathtt{opt}},z_1}$, $P_j^{\prime}\preceq P_j^{T_j,z_1}$ and $T_j^{\prime}\preceq T_j$.
Because of $T_j$ is obtained from $\mathcal{B}$, we have $\mathcal{W}_{-z}(P_j, T_j)^2 < \beta\mathtt{Cost}^{k}_{-z}(P_j)$, From Claim \ref{cl-m1} we know that $\mathtt{Cost}^{k}_{-z}(P)\leq \mathcal{W}_{-z_1,-z_2}(P, S_{\mathtt{opt}})$, then (\ref{for-lem-9-1}) implies $\mathcal{W}_{-(z_1+z_2)}(S_{\mathtt{opt}}, T_j^{\prime}) \leq (1+\sqrt{\beta}) \mathcal{W}_{-z_1,-z_2}(P_j, S_{\mathtt{opt}})$.
\end{proof}

\begin{lemma}
\label{lem-10}
For any $1\leq j \leq m$, the weighted point set $T_j^{\prime}$ in Lemma~\ref{lem-9} also satisfies   
$\mathtt{Cost}_{-(2z_1+z_2), -z_2}(\mathbb{P}, T^{\prime}_j) \leq (2+\sqrt{\beta})\mathtt{Cost}_{-z_1, -z_2}(\mathbb{P}, S_{\mathtt{opt}}) + (2+3\sqrt{\beta}+\beta) \mathcal{W}_{-z_1, -z_2}^2(P_j, S_{\mathtt{opt}}) $.
\end{lemma}

\begin{proof}
For any $1\leq j_1\leq m$, Let $S_{\mathtt{opt}}^{j_1\prime}$ be an arbitrary weighted point set that satisfies  $S_{\mathtt{opt}}^{j_1\prime} \preceq S_{\mathtt{opt}}^{T_j^{\prime},z_1+z_2}\cap S_{\mathtt{opt}}^{P_{j_1},z_2}$ and $w_{S_{\mathtt{opt}}^{j_1\prime}}(S_{\mathtt{opt}}^{j_1\prime}) = n - 2z_1 - z_2$, we have
\begin{eqnarray}
&&\mathcal{W}_{-(2z_1+z_2), -z_2}(P_{j_1}, T^{\prime}_j) \nonumber\\
&\leq& \mathcal{W}_{-(2z_1+z_2)}(P_{j_1}, S_{\mathtt{opt}}^{j_1\prime}) + \mathcal{W}_{-z_2}(T^{\prime}_j, S_{\mathtt{opt}}^{j_1\prime}) \nonumber\\
&\leq& \mathcal{W}_{-z_1}(P_{j_1}, S_{\mathtt{opt}}^{P_{j_1},z_2}) + \mathcal{W}(T^{\prime}_j, S_{\mathtt{opt}}^{T_j^{\prime},z_1+z_2}) \nonumber\\
&=& \mathcal{W}_{-z_1,-z_2}(P_{j_1}, S_{\mathtt{opt}}) + \mathcal{W}_{-(z_1+z_2)}(S_{\mathtt{opt}}, T^{\prime}_j) \nonumber\\
&\leq& \mathcal{W}_{-z_1,-z_2}(P_{j_1}, S_{\mathtt{opt}}) \nonumber\\
&& +\; (1+\sqrt{\beta})\mathcal{W}_{-z_1,-z_2}(P_j, S_{\mathtt{opt}}) \label{for-lem-10-1}
\end{eqnarray}
where the first inequality follows from (\ref{triangle-inequality-for-rwb-1}), the second inequality follows from (\ref{inequality-for-rwb-1}) and the last inequality follows from Lemma~\ref{lem-9}. Then we obtain the following bound by applying (\ref{for-lem-10-1}):
\begin{eqnarray}
&&\mathtt{Cost}_{-(2z_1+z_2), -z_2}(\mathbb{P}, T^{\prime}_j) \nonumber\\
&=& \frac{1}{m}\sum^m_{j_1=1}\mathcal{W}_{-(2z_1+z_2), -z_2}^2(P_{j_1}, T^{\prime}_j) \nonumber\\
&\leq& \frac{1}{m}\sum^m_{j_1=1}(\mathcal{W}_{-z_1,-z_2}(P_{j_1}, S_{\mathtt{opt}}) \nonumber\\
&& +\; (1+\sqrt{\beta})\mathcal{W}_{-z_1,-z_2}(P_j, S_{\mathtt{opt}}))^2 \nonumber\\
&\leq& \frac{1}{m}\sum^m_{j_1=1}(2+\sqrt{\beta})\mathcal{W}_{-z_1,-z_2}^2(P_{j_1}, S_{\mathtt{opt}})\nonumber\\
&& +\;(2+3\sqrt{\beta}+\beta)\mathcal{W}_{-z_1,-z_2}^2(P_j, S_{\mathtt{opt}}) \nonumber\\
&=& (2+\sqrt{\beta})\mathtt{Cost}_{-z_1,-z_2}(\mathbb{P}, S_{\mathtt{opt}}) \nonumber\\
&& +\; (2+3\sqrt{\beta}+\beta)\mathcal{W}_{-z_1,-z_2}^2(P_j, S_{\mathtt{opt}}). \label{for-lem-10-2}
\end{eqnarray}
where the third inequality follows from the fact that $ (a+\lambda b)^2\leq (1+\lambda)a^2+(\lambda^2+\lambda)b^2$ for any real numbers $a, b$, and $\lambda$.
\end{proof}

\begin{proof} 

[\textbf{of Theorem~\ref{the-result1-2-2side}}]
Because of $\tilde{T}_j$ is the optimal weight distribution over $T_j$, so we have $\mathtt{Cost}_{-(2z_1+z_2), -z_2}(\mathbb{P}, \tilde{T}_{j})\leq \mathtt{Cost}_{-(2z_1+z_2), -z_2}(\mathbb{P}, T_{j}^{\prime})$. For the best candidate $\tilde{T}_{j_0}$, we have
\begin{eqnarray}
&&\mathtt{Cost}_{-(2z_1+z_2), -z_2}(\mathbb{P}, \tilde{T}_{j_0}) \nonumber\\
&\leq& \min_{1\leq j\leq m} \mathtt{Cost}_{-(2z_1+z_2), -z_2}(\mathbb{P}, T_{j}^{\prime}) \nonumber\\
&\leq& (2+\sqrt{\beta})\mathtt{Cost}_{-z_1,-z_2}(\mathbb{P}, S_{\mathtt{opt}}) \nonumber\\
&& +\; (2+3\sqrt{\beta}+\beta)\min_{1\leq j\leq m}\mathcal{W}_{-z_1,-z_2}^2(P_j, S_{\mathtt{opt}})  \nonumber\\ 
&\leq& (2+\sqrt{\beta})^2\cdot\mathtt{Cost}_{-z_1,-z_2}(\mathbb{P}, S_{\mathtt{opt}}) , 
\end{eqnarray}
where the second inequality follows from Lemme~\ref{lem-10} and the last inequality follows from 
$$\min_{1\leq j\leq m}\mathcal{W}_{-z_1,-z_2}^2(P_j, S_{\mathtt{opt}}) \leq \mathtt{Cost}_{-z_1,-z_2}(\mathbb{P}, S_{\mathtt{opt}})$$
\end{proof}

\section{Results for any \texorpdfstring{$l\geq 1$}{Lg}}
When $l \neq 2$, we need to use alternative base algorithms instead of $k$-means. For example, when $l = 1$, we should use $k$-median algorithm, and for $l > 2$, we use the $(k,l)$-clustering algorithm~\cite{cohen2021new}. In these cases, we can still derive similar conclusions. If we use an $(k,l)$-clustering algorithm $\mathcal{A}$ with an approximation ratio of $\alpha$, under our framework, we can obtain a solution with an approximation ratio of $O(2^{2l-2}(\alpha + 1))$. Similarly, if we use an $(k,l)$-clustering with outliers algorithm $\mathcal{B}$ with an approximation ratio of $\beta$, under our framework, we can obtain a solution with an approximation ratio of  $O(2^{2l-2}(\beta + 1))$.

Below, we present a variant of Theorem 1 for the case when $l \neq 2$. We provide the proof for this variant, and the proofs for the remaining theorems can be modified similarly. 
%Here, we will not elaborate on them.
\begin{theorem}
\label{the-result1-l!2}
Our clustering based LP Algorithm returns a solution $\tilde{T}_{j_0}$ for $k$-sparse WB with outliers for any $l\geq 1$ and achieves the following quality guarantee:
\begin{eqnarray}
&&\mathtt{Cost}_{-z}(\mathbb{P}, \tilde{T}_{j_0}) \nonumber\\
&\leq& (2^{l-1}+2^{2l-2}(\alpha + 1))\cdot\mathtt{Cost}_{-z}(\mathbb{P}, S_{\mathtt{opt}}).
\end{eqnarray}
\end{theorem}
\begin{proof} 
First, it is easy to see that for each $1\leq j\leq m$, let the obtained cluster centers $ T_j=\{t_1, t_2, \cdots, t_{\lambda(k+\hat{z})}\}$ from $\mathcal{A}$; also each weight $w_{T_j}(t_s)=$ the total weight of the $s$-th cluster. Then  
\begin{eqnarray}
\mathcal{W}_{-z}(T_j, S_{\mathtt{opt}})\leq (1+\sqrt[l]{\alpha}) \mathcal{W}_{-z}(P_j, S_{\mathtt{opt}}).
\end{eqnarray}
This is because $\mathcal{W}^l(T_j, P_j) \leq \alpha\mathcal{W}_{-z}^l(P_j, S_{\mathtt{opt}})$ and $\mathcal{W}_{-z}(T_j, S_{\mathtt{opt}})\leq \mathcal{W}_{-z}(P_j, S_{\mathtt{opt}}) + \mathcal{W}(T_j, P_j)$.
For any $1\leq j_1\leq m$, we have
\begin{eqnarray}
&&\mathcal{W}_{-z}(P_{j_1}, T^{S_{\mathtt{opt}}}_j)
% &=\mathcal{W}(P^{T^{S_{\mathtt{opt}}}_j}_{j_1}, T^{S_{\mathtt{opt}}}_j)\\
\leq \mathcal{W}(P^{S_{\mathtt{opt}}}_{j_1}, T^{S_{\mathtt{opt}}}_j)\nonumber\\
&\leq& \mathcal{W}(P^{S_{\mathtt{opt}}}_{j_1}, S_{\mathtt{opt}})+\mathcal{W}(S_{\mathtt{opt}}, T^{S_{\mathtt{opt}}}_j) \nonumber\\
&=& \mathcal{W}_{-z}(P_{j_1}, S_{\mathtt{opt}})+\mathcal{W}_{-z}(T_j, S_{\mathtt{opt}})\nonumber\\
&\leq& \mathcal{W}_{-z}(P_{j_1}, S_{\mathtt{opt}})+(1+\sqrt[l]{\alpha})\mathcal{W}_{-z}(P_j, S_{\mathtt{opt}}), \label{for-the-6-1}
\end{eqnarray}
where the second inequality follows from the triangle inequality of Wasserstein distance and the third inequality follows from Lemma~\ref{lem-1}. From Corollary A.1 of~\cite{makarychev2019performance}, we know that for every $a\geq 0,b\geq 0,0\leq t \leq 1, l > 1$, we have
\begin{eqnarray}
(a+b)^l\leq (1+t)^{l-1}a^l+(1+\frac{1}{t})^{l-1}b^l \label{triangle-l}
\end{eqnarray}
Then we obtain the following bound by applying (\ref{for-the-6-1})  and set $t=1$ in (\ref{triangle-l}):
% \end{eqnarray}
\begin{small}
\begin{eqnarray}
&&\mathtt{Cost}_{-z}(\mathbb{P}, T^{S_{\mathtt{opt}}}_j) \nonumber\\
% &=& \frac{1}{m}\sum^m_{j_1=1}\mathcal{W}_{-z}(P_{j_1}, T^{S_{\mathtt{opt}}}_j) \nonumber\\
&\leq& \frac{1}{m}\sum^m_{j_1=1}(\mathcal{W}_{-z}(P_{j_1}, S_{\mathtt{opt}})+(1+\sqrt[l]{\alpha})\mathcal{W}_{-z}(P_j, S_{\mathtt{opt}}))^l \nonumber\\
&\leq& \frac{1}{m}\sum^m_{j_1=1}2^{l-1}\mathcal{W}_{-z}^{l}(P_{j_1}, S_{\mathtt{opt}})\nonumber\\
&& +\;2^{l-1}(1+\sqrt[l]{\alpha})^{l}\mathcal{W}_{-z}^{l}(P_j, S_{\mathtt{opt}}) \nonumber\\
&\leq& 2^{l-1}\mathtt{Cost}_{-z}(\mathbb{P}, S_{\mathtt{opt}}) + 2^{2l-2}(1+\alpha)\mathcal{W}_{-z}^{l}(P_j, S_{\mathtt{opt}}) \label{for-the-6-2}
\end{eqnarray}
\end{small}
where the second inequality and the last inequality follows from (\ref{triangle-l}).

Because of $\tilde{T}_j$ is the optimal weight distribution over $T_j$, so we have $\mathtt{Cost}_{-z}(\mathbb{P}, \tilde{T}_j)\leq \mathtt{Cost}_{-z}(\mathbb{P}, T^{S_{\mathtt{opt}}}_j)$. For the best candidate $\tilde{T}_{j_0}$, we have
\begin{small}
\begin{eqnarray}
&& \mathtt{Cost}_{-z}(\mathbb{P}, \tilde{T}_{j_0}) \nonumber\\
&\leq& \min_{1\leq j\leq m}\mathtt{Cost}_{-z}(\mathbb{P}, T^{S_{\mathtt{opt}}}_j) \nonumber\\
&\leq& 2^{l-1}\mathtt{Cost}_{-z}(\mathbb{P}, S_{\mathtt{opt}}) + 2^{2l-2}(1+\alpha)\min_{1\leq j\leq m}\mathcal{W}_{-z}^{l}(P_j, S_{\mathtt{opt}}) \nonumber\\ 
&\leq& (2^{l-1}+2^{2l-2}(\alpha + 1))\cdot\mathtt{Cost}_{-z}(\mathbb{P}, S_{\mathtt{opt}})
\end{eqnarray}
\end{small}
where the second inequality follows from (\ref{for-the-6-2}) and the last inequality follows from $\min_{1\leq j\leq m}\mathcal{W}_{-z}^l(P_j, S_{\mathtt{opt}})\leq \mathtt{Cost}_{-z}(\mathbb{P}, S_{\mathtt{opt}})$.
\end{proof}
By using inequality (\ref{triangle-l}), we can modify Theorems \ref{the-result1-2} in a similar manner to hold for any $l$.

\section{Other Detailed Proof}
\subsection{The Proof of Claim \ref{cl-fwd}}

Here we present a more general version of Claim \ref{cl-fwd} for the two-side case. To achieve this, we introduce a dummy point $q_*$ to $Q$ with a weight equal to $z_1$, and a dummy point $p_*$ to $P$ with a weight equal to $z_2$. We enforce the distance between $q_*$ and each $p_i$, as well as the distance between $p_*$ and each $q_i$, to be $0$, while setting the distance between $q_*$ and $p_*$ to $\infty$. Please note that we do not actually find a real point $q_*$ and $p_*$ in the space; instead, we set all the entries corresponding to $p_*$ and $q_*$ in the $(n_1+1)\times (n_2+1)$ distance matrix to be $0$. Computing $\mathcal{W}_{-z_1,-z_2}(P, Q)$ is equivalent to computing $\mathcal{W}(P\cup\{p_*\}, Q\cup\{q_*\})$.

\begin{proof}
First, recall the definition of the Wasserstein distance:
\begin{eqnarray}
\mathcal{W}(P, Q)=\min_{F} \Big(\sum^{n_1}_{i=1}\sum^{n_2}_{j=1}f_{ij}||p_i-q_j||^l\Big)^{\frac{1}{l}} \label{fwd-2}
\end{eqnarray}
For convenience, we also define 
\begin{eqnarray}
\mathcal{W}^F(P, Q)=(\sum^{n_1}_{i=1}\sum^{n_2}_{j=1}f_{ij}||p_i-q_j||^l)^{\frac{1}{l}},
\end{eqnarray}
{\em i.e.,} $\mathcal{W}(P, Q)=\min_{F}\mathcal{W}^F(P, Q)$.
Let 
\begin{eqnarray}
F^*=\arg\min_{F}\mathcal{W}^{F}(P\cup\{p_*\}, Q\cup\{q_*\})
\end{eqnarray}
be the best flows between $P\cup\{p_*\}$ and $Q\cup\{q_*\}$. Since the distance between $p_*$ and $q_*$ is $\infty$, it is clear that the flow between them $f_{**}=0$.

First, for any feasible flow $F$ between $P\cup\{p_*\}$ and $Q\cup\{q_*\}$ satisfying $f_{**}=0$, we can construct a flow 
\begin{eqnarray}
F^{\prime}=\{f_{ij}^{\prime}=f_{ij}\mid 1\leq i\leq n_1, 1\leq j\leq n_2\}
\end{eqnarray}
which is a feasible flow between $\mathbb{P}^P_{-z_1}$ and $\mathbb{P}^Q_{-z_2}$ 
because of 
\begin{eqnarray}
\sum^{n_1}_{i=1} f_{ij}\leq \sum^{n_1}_{i=1} f_{ij}^{\prime} + f_{*j}^{\prime} =w_Q(q_j)
\end{eqnarray}
for any $j$, and 
\begin{eqnarray}
\sum^{n_2}_{j=1} f_{ij} \leq \sum^{n_2}_{j=1} f_{ij}^{\prime} + f_{i*}^{\prime} =w_P(p_i)
\end{eqnarray}
for any $i$, and 
\begin{eqnarray}
\resizebox{.91\linewidth}{!}{$\displaystyle
\sum^{n_1}_{i=1}\sum^{n_2}_{j=1} f_{ij}=\sum^{n_1}_{i=1}\sum^{n_2}_{j=1} f_{ij}^{\prime}=w_{P}(P)-z_1=w_{Q}(Q)-z_2.$}
\end{eqnarray}
Similarly, for any feasible flow $F^{\prime}$ between $\mathbb{P}^P_{-z_1}$ and $\mathbb{P}^Q_{-z_2}$ , we can construct a feasible flow 
\begin{center}
{ $F=\{f_{ij}=f_{ij}^{\prime} \text{ and } f_{i*} = w_{P}(p_i) - \sum_{j=1}^{n_2}f_{ij}^{\prime} \text{ and } f_{*j} = w_{Q}(q_j) - \sum_{i=1}^{n_1}f_{ij}^{\prime} \mid 1\leq i\leq n_1, 1\leq j\leq n_2\}$ }
\end{center}
which is a feasible flow between $P\cup\{p_*\}$ and $Q\cup\{q_*\}$.

Because $||p_i-q_*||=0$,  $||p_*-q_j||=0$, and $f_{ij}=f_{ij}^{\prime}$ for every $i,j$, so we have \begin{eqnarray}
\mathcal{W}^{F}(P\cup\{p_*\}, Q\cup\{q_*\})=\mathcal{W}_{-z_1,-z_2}^{F^{\prime}}(P, Q).
\end{eqnarray}
Since $F^{\prime}$ and $F$ are in one-to-one correspondence, we have 
\begin{eqnarray}
&&\mathcal{W}_{-z_1,-z_2}(P, Q)\nonumber \\
&=&\min_{F} \mathcal{W}^{F}(P\cup\{p_*\}, Q\cup\{q_*\})\nonumber \\
&=&\min_{F^{\prime}} \mathcal{W}_{-z_1,-z_2}^{F^{\prime}}(P, Q)\nonumber \\
&=&\mathcal{W}(P\cup\{p_*\}, Q\cup\{q_*\}). \label{for-rmk-3}
\end{eqnarray}
Therefore, computing $\mathcal{W}_{-z_1,-z_2}(P, Q)$ is equivalent to computing $\mathcal{W}(P\cup\{p_*\}, Q\cup\{q_*\})$.
\end{proof}

\subsection{Formulation of Remark \ref{rem-fixoutlier}}

First, we know that 
$$\mathcal{W}^F(P, Q)^2=\sum^{n_1}_{i=1}\sum^{n_2}_{j=1}f_{ij}||p_i-q_j||^2$$
is a linear function with respect to  $F=\{f_{ij}\mid 1\leq i\leq n_1, 1\leq j\leq n_2\} \in \mathcal{M}(P, Q)$; we use $\mathcal{M}(P, Q)$ to denote the set of feasible flows  from $P$ to $Q$, where each $F \in \mathcal{M}(P, Q)$ satisfies: $\sum^{n_1}_{i=1} f_{ij}=w_Q(q_j)$ for any $j$, and $\sum^{n_2}_{j=1} f_{ij}=w_P(p_i)$ for any $i$. If the support of $S$ be fixed to a given set $G$, the problem of WB is equivalent to
\begin{eqnarray}
&&\min_{w_G}\frac{1}{m}\sum^m_{j=1}\min_{F_j}\mathcal{W}^{F_j}(P_j, G)^2 \nonumber \\
&&\mathrm{s.t.}\text{\;\;\;} F_j \in \mathcal{M}(P_j, G)
\label{for-wblp}
\end{eqnarray}
Clearly, the problem (\ref{for-wblp}) can be solved using linear programming. By referring to (\ref{for-rmk-3}), we know that the fixed-support WB with outliers problem is equivalent to the vanilla fixed-support WB problem with replacing each $P_j$ and $G$ by $P_j\cup\{p_*\}$ and $G \cup \{g_*\}$, respectively. Therefore, the fixed-support WB with outliers problem can also be solved using linear programming.

\subsection{Proof of Claim \ref{cl-m2}}

To prove Claim 3, it suffices to prove the following two arguments: (1) for any feasible solution of the $k$-means clustering that meets the condition that for $1\leq s\leq k$,
\begin{eqnarray}
w _{U _s}(U _s\cap P _1)=w _{U _s}(U _s\cap P _2)=\cdots=w _{U _s}(U _s\cap P _m),\label{for-faircondition}
\end{eqnarray}
we can construct a weighted $k$-point set $S$ such that  $m\mathtt{Cost} _{-z}(\mathbb{P}, S)$ is no greater than the clustering cost. (2) Furthermore, for any weighted $k$-point set $S$, we can construct a clustering result that satisfies the condition (\ref{for-faircondition}), such that the clustering cost is no larger than  $m\mathtt{Cost} _{-z}(\mathbb{P}, S)$.

We prove the  argument (1) first.  We know that the clustering cost is $\sum _{s=1}^{k}\sum _{p\in U _s}w _{U _s}(p)||p-c _s||^2$. We directly set $S=\{c _1,\dots,c _k\}$ with $w _S(c _s)=w _{U _s}(U _s\cap P _j)$ for $1\leq s\leq k$, and simultaneously, we construct a feasible flow $F _j$ from $P _j$ to $S$ as $f^j _{is}=w _{U _s}(p _{ji})$ if $p _{ji}\in U _s$, else $f^j _{is}=0$. We can prove that 
$m\mathtt{Cost} _{-z}(\mathbb{P}, S)\leq\sum _{j=1}^m \sum _{i=1}^{n _j} \sum _{s=1}^k f^j_{is}||p _{ji}-c _s|| = \sum _{s=1}^{k}\sum _{p\in U _s}w _{U _s}(p)||p-c _s||^2$.

Similarly, for each given $S=\{c _1,\dots,c _k\}$, we can obtain the optimal transport flow $F _j=\{f^j_{is}\mid 1\leq i\leq |P_j|, 1\leq s\leq k\}$ induced by $\mathtt{Cost} _{-z}(\mathbb{P}, S)$ for each $1\leq j\leq m$. We can construct a clustering solution $\{U _1,\dots,U _k\}$ satisfying that the center of each $U _s$ is $c _s$, and $w _{U _s}(p _{ji})=f^j_{is}$. Clearly, such a solution conforms to the condition (\ref{for-faircondition}). It is also easy to know that:
$ \sum _{s=1}^{k}\sum _{p\in U _s}w _{U _s}(p)||p-c _s||^2 = \sum _{j=1}^m \sum _{i=1}^{n _j} \sum _{s=1}^k f^j_{is}||p _{ji}-c _s|| = m\mathtt{Cost} _{-z}(\mathbb{P}, S)$. Therefore, the argument (2) is true. 

Overall, computing the optimal solution of $k$-sparse WB with $z$ outliers is equivalent to solving the  $k$-means clustering with $mz$ outliers on $\cup^m _{j=1}P _j$ with the  condition (\ref{for-faircondition}).

\subsection{The Proof of Claim \ref{cl-for-kzmeans}}
For simplicity of the notations, we use ``$P$'' to replace the notation ``$P_j$'' in our following proof. 
Let $P^k_{-z} \in \mathbb{P}^P_{-z}$ represent the optimal inliers of $P$ with respect to the $k$-means clustering with $z$ outliers on $P$. We  have the following lemma and the proof is shown later. 
\begin{lemma}
\label{lm-sparsity}
$|\mathtt{supp}(P\setminus P^{k}_{-z})|\leq\hat{z}$.
\end{lemma}

% {\color{red}We also have $|\mathtt{supp}(P\setminus P^{k}_{-z})|\leq\hat{z}$ (if this is not true, we can always convert the solution $P^k_{-z}$ to have $|\mathtt{supp}(P\setminus P^{k}_{-z})|\leq\hat{z}$ by simple adjustments on the weights without affecting the clustering quality; we show the details in section \ref{sec-sparsity}).}

Then we have $\mathtt{Mean}^{k+\hat{z}}(P)\leq \mathtt{Mean}^{k}_{-z}(P)$ for any set $P$, because 
\begin{eqnarray}
&&\mathtt{Mean}^{k+\hat{z}}(P) 
= \mathcal{W}(P, \mathtt{C}^{k+\hat{z}}(P))^2 \nonumber\\
&\leq& \mathcal{W}(P, \mathtt{C}^{k}_{-z}(P)\cup(P\setminus P^{k}_{-z}))^2\nonumber\\
&\leq& \mathcal{W}(P^{k}_{-z}, \mathtt{C}^{k}_{-z}(P))^2 + \mathcal{W}(P\setminus P^{k}_{-z}, P\setminus P^{k}_{-z})^2\nonumber\\
&=& \mathtt{Mean}^{k}_{-z}(P). \label{for-lem-1-1}
\end{eqnarray}
The key is to show the correctness of the second inequality:
if we specify the centers of $P^{k}_{-z}$ as $\mathtt{C}^{k}_{-z}(P)$ and the centers of $P\setminus P^{k}_{-z}$ as $P\setminus P^{k}_{-z}$ (Lemma~\ref{lm-sparsity} tells us that $P\setminus P^{k}_{-z}$ actually can be viewed as at most $\hat{z}$ separate points, where each point can 
 be regarded as a cluster with a single point); the obtained cost is no smaller than the optimal cost of the $(k+\hat{z})$-means clustering on $P$.

\begin{table*}[!ht]
\small
  \centering
  \begin{tabular}{c@{\hspace{5pt}}c@{\hspace{5pt}}c@{\hspace{10pt}}|@{\hspace{10pt}}c@{\hspace{5pt}}c@{\hspace{5pt}}c@{\hspace{5pt}}c@{\hspace{5pt}}c@{\hspace{5pt}}c@{\hspace{5pt}}c@{\hspace{10pt}}|@{\hspace{10pt}}c@{\hspace{5pt}}c@{\hspace{5pt}}c@{\hspace{5pt}}c@{\hspace{5pt}}c@{\hspace{5pt}}c@{\hspace{5pt}}c} %@{\hspace{5pt}}
  \toprule 
    & & & \multicolumn{14}{c}{\textbf{Proportion of Outliers $z/n$}} \\
    & & & \multicolumn{6}{c}{\textbf{Our\_$\mathcal{A}$}}& & \multicolumn{7}{c}{\textbf{Our\_$\mathcal{B}$}} \\
    $m$&$d$&$k$ & $0$   & $0.025$  & $0.05$  & $0.075$ & $0.1$ &$0.125$ &$0.15$& $0$   & $0.025$  & $0.05$  & $0.075$ & $0.1$ &$0.125$ &$0.15$  \\ 
    \cmidrule(r){1-17}
    2 &10 &10 & 1.010 & 1.183 & 1.092 & 1.270 & 1.141 & 1.072 & 1.139 & 1.009 & 1.111 & 1.327 & 1.654 & 1.984 & 2.131 & 3.040 \\
    2 &10 &20 & 1.186 & 1.005 & 1.057 & 1.074 & 1.073 & 1.104 & 1.174 & 0.979 & 1.121 & 1.264 & 1.458 & 1.755 & 2.659 & 2.807 \\
    2 &10 &30 & 1.001 & 1.198 & 1.028 & 1.087 & 1.103 & 1.218 & 1.164 & 1.025 & 1.079 & 1.266 & 1.648 & 1.656 & 2.703 & 2.609 \\
    2 &10 &40 & 1.015 & 1.036 & 1.060 & 1.071 & 1.105 & 1.161 & 1.185 & 0.982 & 1.160 & 1.431 & 1.840 & 2.061 & 2.263 & 2.873 \\
    2 &20 &10 & 1.009 & 1.114 & 1.440 & 1.911 & 1.870 & 1.039 & 1.058 & 1.027 & 1.254 & 1.683 & 2.202 & 2.942 & 3.233 & 3.156 \\
    2 &20 &20 & 1.198 & 1.101 & 1.124 & 1.061 & 1.064 & 1.110 & 1.111 & 1.008 & 1.176 & 1.601 & 2.556 & 3.116 & 3.275 & 3.108 \\
    2 &20 &30 & 1.193 & 1.108 & 1.053 & 1.047 & 1.103 & 1.151 & 1.169 & 1.004 & 1.181 & 1.644 & 2.095 & 2.859 & 3.020 & 3.172 \\
    2 &20 &40 & 1.002 & 1.038 & 1.064 & 1.090 & 1.137 & 1.155 & 1.240 & 1.001 & 1.192 & 1.570 & 2.355 & 2.793 & 2.667 & 3.270 \\
    2 &30 &10 & 1.003 & 1.143 & 1.914 & 2.814 & 2.950 & 1.017 & 1.195 & 1.005 & 1.346 & 2.167 & 2.438 & 3.012 & 2.838 & 2.978 \\
    2 &30 &20 & 1.024 & 1.185 & 2.069 & 1.012 & 1.018 & 1.009 & 1.014 & 0.989 & 1.334 & 2.918 & 3.023 & 2.407 & 2.934 & 2.565 \\
    2 &30 &30 & 1.004 & 1.207 & 1.004 & 1.013 & 1.023 & 1.024 & 1.045 & 1.023 & 1.388 & 2.466 & 2.586 & 3.158 & 3.193 & 3.091 \\
    2 &30 &40 & 1.199 & 1.156 & 1.026 & 1.038 & 1.041 & 1.041 & 1.057 & 0.986 & 1.330 & 3.046 & 3.137 & 3.211 & 2.910 & 3.241 \\
    2 &40 &10 & 1.003 & 1.482 & 2.663 & 3.052 & 3.157 & 1.008 & 1.200 & 1.005 & 1.448 & 3.107 & 2.676 & 2.801 & 3.159 & 3.136 \\
    2 &40 &20 & 1.200 & 1.318 & 3.287 & 1.198 & 1.008 & 1.193 & 1.014 & 1.017 & 1.448 & 2.798 & 3.299 & 2.533 & 2.595 & 2.877 \\
    2 &40 &30 & 1.005 & 1.339 & 1.004 & 1.000 & 1.008 & 1.003 & 1.198 & 0.999 & 1.751 & 3.022 & 2.889 & 2.599 & 3.264 & 3.000 \\
    2 &40 &40 & 1.007 & 2.017 & 1.013 & 1.010 & 1.014 & 1.009 & 1.015 & 0.991 & 2.722 & 3.175 & 3.075 & 3.130 & 3.048 & 2.977 \\
    \cmidrule(r){1-17}
    4 &10 &10 & 1.217 & 1.229 & 1.346 & 1.457 & 1.491 & 1.297 & 1.404 & 1.168 & 1.327 & 1.532 & 1.705 & 1.972 & 2.379 & 2.767 \\
    4 &10 &20 & 1.242 & 1.256 & 1.260 & 1.303 & 1.343 & 1.353 & 1.409 & 1.233 & 1.354 & 1.524 & 1.798 & 2.164 & 2.900 & 2.982 \\
    4 &10 &30 & 1.232 & 1.248 & 1.241 & 1.292 & 1.333 & 1.365 & 1.420 & 1.252 & 1.378 & 1.710 & 1.947 & 2.209 & 2.217 & 2.972 \\
    4 &10 &40 & 1.234 & 1.240 & 1.259 & 1.292 & 1.355 & 1.394 & 1.404 & 1.234 & 1.357 & 1.513 & 2.042 & 2.451 & 2.577 & 3.149 \\
    4 &20 &10 & 1.206 & 1.368 & 1.617 & 2.552 & 2.875 & 1.259 & 1.236 & 1.203 & 1.461 & 1.800 & 2.786 & 3.043 & 3.152 & 2.637 \\
    4 &20 &20 & 1.239 & 1.315 & 1.408 & 1.268 & 1.302 & 1.331 & 1.370 & 1.213 & 1.404 & 1.876 & 2.689 & 2.979 & 2.871 & 2.916 \\
    4 &20 &30 & 1.224 & 1.332 & 1.269 & 1.314 & 1.330 & 1.380 & 1.391 & 1.219 & 1.408 & 1.807 & 2.370 & 3.191 & 2.712 & 2.962 \\
    4 &20 &40 & 1.227 & 1.265 & 1.300 & 1.339 & 1.370 & 1.453 & 1.462 & 1.228 & 1.421 & 2.211 & 2.594 & 2.543 & 3.154 & 3.017 \\
    4 &30 &10 & 1.231 & 1.454 & 2.396 & 2.678 & 1.234 & 1.240 & 1.239 & 1.203 & 1.587 & 2.623 & 2.872 & 3.227 & 2.856 & 2.873 \\
    4 &30 &20 & 1.224 & 1.468 & 1.393 & 1.245 & 1.250 & 1.254 & 1.246 & 1.234 & 1.540 & 2.779 & 2.960 & 2.670 & 2.506 & 2.681 \\
    4 &30 &30 & 1.234 & 1.496 & 1.234 & 1.247 & 1.257 & 1.265 & 1.271 & 1.237 & 1.515 & 2.966 & 2.988 & 2.718 & 3.271 & 2.764 \\
    4 &30 &40 & 1.243 & 1.334 & 1.271 & 1.263 & 1.289 & 1.304 & 1.320 & 1.242 & 2.031 & 3.066 & 3.139 & 3.160 & 2.873 & 2.842 \\
    4 &40 &10 & 1.196 & 1.624 & 3.205 & 3.104 & 2.720 & 1.213 & 1.217 & 1.214 & 1.623 & 2.740 & 3.218 & 2.475 & 3.202 & 2.843 \\
    4 &40 &20 & 1.241 & 1.667 & 2.825 & 1.224 & 1.234 & 1.244 & 1.220 & 1.233 & 1.754 & 3.216 & 3.055 & 2.850 & 2.810 & 2.766 \\
    4 &40 &30 & 1.229 & 1.954 & 1.245 & 1.234 & 1.244 & 1.253 & 1.243 & 1.236 & 1.768 & 2.833 & 3.197 & 2.702 & 2.885 & 3.083 \\
    4 &40 &40 & 1.235 & 1.665 & 1.253 & 1.251 & 1.248 & 1.260 & 1.256 & 1.227 & 2.030 & 3.007 & 2.949 & 2.750 & 3.272 & 2.687 \\
    \cmidrule(r){1-17}
    6 &10 &10 & 1.247 & 1.270 & 1.400 & 1.428 & 1.453 & 1.385 & 1.402 & 1.284 & 1.347 & 1.613 & 1.804 & 2.113 & 2.829 & 2.858 \\
    6 &10 &20 & 1.272 & 1.264 & 1.305 & 1.334 & 1.360 & 1.405 & 1.484 & 1.305 & 1.398 & 1.812 & 2.394 & 2.055 & 2.917 & 3.231 \\
    6 &10 &30 & 1.302 & 1.314 & 1.330 & 1.335 & 1.393 & 1.415 & 1.434 & 1.287 & 1.405 & 1.559 & 2.140 & 2.504 & 2.577 & 3.120 \\
    6 &10 &40 & 1.290 & 1.311 & 1.332 & 1.358 & 1.400 & 1.427 & 1.471 & 1.322 & 1.444 & 1.590 & 2.047 & 2.594 & 2.779 & 3.265 \\
    6 &20 &10 & 1.267 & 1.382 & 1.698 & 2.607 & 1.284 & 1.266 & 1.302 & 1.276 & 1.503 & 1.950 & 2.687 & 3.251 & 3.109 & 2.756 \\
    6 &20 &20 & 1.281 & 1.375 & 1.345 & 1.331 & 1.351 & 1.366 & 1.402 & 1.292 & 1.503 & 1.914 & 3.011 & 2.962 & 3.285 & 2.917 \\
    6 &20 &30 & 1.294 & 1.401 & 1.348 & 1.355 & 1.405 & 1.440 & 1.450 & 1.312 & 1.500 & 1.953 & 3.008 & 3.050 & 3.248 & 2.949 \\
    6 &20 &40 & 1.296 & 1.360 & 1.354 & 1.391 & 1.410 & 1.473 & 1.518 & 1.313 & 1.552 & 1.858 & 2.334 & 2.648 & 3.241 & 3.242 \\
    6 &30 &10 & 1.277 & 1.533 & 2.142 & 3.109 & 2.815 & 1.279 & 1.275 & 1.270 & 1.565 & 2.790 & 3.214 & 3.136 & 3.081 & 3.053 \\
    6 &30 &20 & 1.283 & 1.484 & 1.595 & 1.308 & 1.310 & 1.311 & 1.307 & 1.325 & 1.599 & 3.137 & 2.938 & 2.691 & 2.565 & 3.014 \\
    6 &30 &30 & 1.297 & 1.502 & 1.319 & 1.320 & 1.320 & 1.330 & 1.345 & 1.315 & 1.589 & 3.116 & 3.107 & 3.103 & 2.920 & 3.251 \\
    6 &30 &40 & 1.302 & 1.596 & 1.328 & 1.345 & 1.351 & 1.353 & 1.362 & 1.319 & 1.738 & 3.136 & 3.048 & 3.026 & 2.852 & 3.254 \\
    6 &40 &10 & 1.280 & 1.875 & 2.998 & 3.137 & 3.246 & 1.266 & 1.280 & 1.295 & 1.718 & 3.090 & 3.085 & 2.930 & 3.203 & 3.240 \\
    6 &40 &20 & 1.289 & 1.749 & 2.410 & 1.289 & 1.307 & 1.298 & 1.299 & 1.321 & 1.777 & 2.866 & 2.972 & 3.176 & 3.126 & 2.685 \\
    6 &40 &30 & 1.296 & 1.655 & 1.314 & 1.312 & 1.307 & 1.304 & 1.317 & 1.327 & 1.929 & 3.210 & 2.708 & 2.650 & 2.935 & 2.480 \\
    6 &40 &40 & 1.310 & 1.510 & 1.311 & 1.322 & 1.325 & 1.321 & 1.327 & 1.314 & 1.824 & 3.044 & 3.274 & 3.117 & 2.806 & 2.901 \\
    \bottomrule
\end{tabular}
\caption{The approximation ratios of our algorithm for $m=2$-$6$.}
\label{table:synthetic_sp1}
\end{table*}

\begin{proof}

[\textbf{of Lemma~\ref{lm-sparsity}}]
% The sparsity of the solution for $k$-means with outliers refers to the property that only a small number of data points are assigned to the outlier cluster. Let $P^k_{-z} \in \mathbb{P}^P_{-z}$ represent the optimal inliers of $P$ with respect to the $k$-means clustering with $z$ outliers on $P$. We want to prove that there always exists a set $P^k_{-z}$ who is an optimal inliers set of $P$ that satisfies $|\mathtt{supp}(P\setminus P^{k}_{-z})|\leq\hat{z}$.
To prove the above conclusion, it is sufficient to prove the existence of a set $P^k_{-z}$, which is an optimal inliers subset of $P$, such that there exists at most one $p_i\in P$   not satisfying ``$w_{P^k_{-z}}(p_i)=0$ or $w_{P^k_{-z}}(p_i)=w_{P}(p_i)$'' (for convenience, we name the point who does not satisfy this condition as a ``bad point'' ). \textbf{Namely, we want to prove that in $P\setminus P^k_{-z}$, there is at most one bad point.} Suppose this claim is true; then, since the weights of the points in $P\setminus P^k_{-z}$ are at least $w_{\mathtt{min}}$, and the total weight of $w_{P\setminus P^k_{-z}}(P\setminus P^k_{-z})=z$, we have $\mathtt{supp}(P\setminus P^k_{-z}) \leq \lceil z/w_{\mathtt{min}}\rceil =\hat{z}$.

Next, we use a constructive approach to prove the above claim. Suppose there exists an optimal solution $P^k_{-z}$ for $k$-means with outliers, where more than two bad points. Then, we can construct a modified solution $(P^{k}_{-z})^{\prime}$ from $P^k_{-z}$ using a step-by-step adjustment method, such that $(P^{k}_{-z})^{\prime}$ satisfies the given condition.

First, we arbitrarily choose two bad points $p_{i_1}$ and $ p_{i_2}$ from $P^k_{-z}$. Because   $P^k_{-z}$ is an optimal solution for $k$-means with outliers, it is clear that we have: 

% \begin{center}
% \fbox{\em $\mathtt{d}(p_{i_1},\mathtt{C}^{k}_{-z}(P))=\mathtt{d}(p_{i_2},\mathtt{C}^{k}_{-z}(P))$ where $\mathtt{C}^{k}_{-z}(P)$ is the set of optimal cluster centers. }
% \end{center}
\begin{center}
{\em $\mathtt{d}(p_{i_1},\mathtt{C}^{k}_{-z}(P))=\mathtt{d}(p_{i_2},\mathtt{C}^{k}_{-z}(P))$ where $\mathtt{C}^{k}_{-z}(P)$ is the set of optimal cluster centers. }
\end{center}

We  prove the above statement by contradiction. 
We know that  $\mathtt{Mean}^{k}_{-z}(P)=\mathcal{S}(P^k_{-z}, \mathtt{C}^{k}_{-z}(P))$. If $\mathtt{d}(p_{i_1},\mathtt{C}^{k}_{-z}(P)) \neq \mathtt{d}(p_{i_2},\mathtt{C}^{k}_{-z}(P))$, without loss of generality, we assume  $\mathtt{d}(p_{i_1},\mathtt{C}^{k}_{-z}(P)) < \mathtt{d}(p_{i_2},\mathtt{C}^{k}_{-z}(P))$. 
Then we let $\epsilon = \min(w_{P}(p_{i_1}) - w_{P^{k}_{-z}}(p_{i_1}),w_{P^{k}_{-z}}(p_{i_2}))$, and can construct $(P^{k}_{-z})^{\prime}$ satisfying 
$w_{(P^{k}_{-z})^{\prime}}(p_{i_1}) = w_{P^{k}_{-z}}(p_{i_1}) + \epsilon, w_{(P^{k}_{-z})^{\prime}}(p_{i_2}) = w_{P^{k}_{-z}}(p_{i_2}) - \epsilon$ 
and $w_{(P^{k}_{-z})^{\prime}}(p_{i}) = w_{P^{k}_{-z}}(p_{i})$ for $i\neq i_1, i_2$; as a consequence,  
\begin{small}
\begin{eqnarray}
&&\mathcal{S}((P^{k}_{-z})^{\prime}, \mathtt{C}^{k}_{-z}(P)) \nonumber \\
&=& \mathcal{S}(P^k_{-z}, \mathtt{C}^{k}_{-z}(P)) - \epsilon( \mathtt{d}(p_{i_2},\mathtt{C}^{k}_{-z}(P)) - \mathtt{d}(p_{i_1},\mathtt{C}^{k}_{-z}(P))) \nonumber \\
&<& \mathcal{S}(P^k_{-z}, \mathtt{C}^{k}_{-z}(P)).
\end{eqnarray}
\end{small}
This contradicts the assumption that $P^k_{-z}$ is an optimal solution. So $\mathtt{d}(p_{i_1},\mathtt{C}^{k}_{-z}(P)) = \mathtt{d}(p_{i_2},\mathtt{C}^{k}_{-z}(P))$. 
We can apply the same method to construct $(P^{k}_{-z})^{\prime}$, satisfying $w_{(P^{k}_{-z})^{\prime}}(p_{i_1}) = w_{P^{k}_{-z}}(p_{i_1}) + \epsilon, w_{(P^{k}_{-z})^{\prime}}(p_{i_2}) = w_{P^{k}_{-z}}(p_{i_2}) - \epsilon$ 
and $w_{(P^{k}_{-z})^{\prime}}(p_{i}) = w_{P^{k}_{-z}}(p_{i})$ for $i\neq i_1, i_2$. We have
\begin{small}
\begin{eqnarray}
&&\mathcal{S}((P^{k}_{-z})^{\prime}, \mathtt{C}^{k}_{-z}(P)) \nonumber \\
&=& \mathcal{S}(P^k_{-z}, \mathtt{C}^{k}_{-z}(P)) - \epsilon( \mathtt{d}(p_{i_2},\mathtt{C}^{k}_{-z}(P)) - \mathtt{d}(p_{i_1},\mathtt{C}^{k}_{-z}(P))) \nonumber \\
&=& \mathcal{S}(P^k_{-z}, \mathtt{C}^{k}_{-z}(P)) = \mathtt{Mean}^{k}_{-z}(P).
\end{eqnarray}
\end{small}
Therefore, $(P^{k}_{-z})^{\prime}$ is also an optimal solution for the $k$-means with outliers problem. Because of $\epsilon = \min(w_{w_{P}}(p_{i_1}) - w_{P^{k}_{-z}}(p_{i_1}),w_{P^{k}_{-z}}(p_{i_2}))$,  satisfies $w_{(P^{k}_{-z})^{\prime}}(p_{i_2})=0$ or $w_{(P^{k}_{-z})^{\prime}}(p_{i_1})=w_{P}(p_{i_1})$. By this way, the number of bad points is reduced by one. By applying similar adjustments step by step, we can obtain a $P^{k}_{-z}$, such that there is at most one bad point. Then we have $|\mathtt{supp}(P\setminus P^{k}_{-z})|\leq\hat{z}$.
\end{proof}

\begin{table*}[!ht]
\small
  \centering
  \begin{tabular}{c@{\hspace{5pt}}c@{\hspace{5pt}}c@{\hspace{10pt}}|@{\hspace{10pt}}c@{\hspace{5pt}}c@{\hspace{5pt}}c@{\hspace{5pt}}c@{\hspace{5pt}}c@{\hspace{5pt}}c@{\hspace{5pt}}c@{\hspace{10pt}}|@{\hspace{10pt}}c@{\hspace{5pt}}c@{\hspace{5pt}}c@{\hspace{5pt}}c@{\hspace{5pt}}c@{\hspace{5pt}}c@{\hspace{5pt}}c} %@{\hspace{5pt}}
  \toprule 
    & & & \multicolumn{14}{c}{\textbf{Proportion of Outliers $z/n$}} \\
    & & & \multicolumn{6}{c}{\textbf{Our\_$\mathcal{A}$}}& & \multicolumn{7}{c}{\textbf{Our\_$\mathcal{B}$}} \\
    $m$&$d$&$k$ & $0$   & $0.025$  & $0.05$  & $0.075$ & $0.1$ &$0.125$ &$0.15$& $0$   & $0.025$  & $0.05$  & $0.075$ & $0.1$ &$0.125$ &$0.15$  \\ 
    \cmidrule(r){1-17}
    8 &10 &10 & 1.308 & 1.309 & 1.412 & 1.458 & 1.636 & 1.437 & 1.516 & 1.429 & 1.554 & 1.697 & 1.745 & 1.757 & 2.113 & 2.482 \\
    8 &10 &20 & 1.330 & 1.336 & 1.440 & 1.409 & 1.422 & 1.482 & 1.544 & 1.491 & 1.534 & 1.509 & 1.686 & 2.067 & 2.731 & 2.691 \\
    8 &10 &30 & 1.329 & 1.364 & 1.392 & 1.399 & 1.431 & 1.521 & 1.539 & 1.559 & 1.425 & 1.689 & 1.904 & 2.019 & 1.945 & 2.681 \\
    8 &10 &40 & 1.372 & 1.374 & 1.403 & 1.455 & 1.454 & 1.528 & 1.545 & 1.564 & 1.537 & 1.503 & 2.044 & 2.310 & 2.375 & 2.858 \\
    8 &20 &10 & 1.309 & 1.424 & 1.881 & 2.047 & 1.703 & 1.350 & 1.365 & 1.447 & 1.590 & 1.865 & 2.738 & 3.006 & 3.152 & 2.647 \\
    8 &20 &20 & 1.352 & 1.493 & 1.505 & 1.385 & 1.426 & 1.438 & 1.491 & 1.540 & 1.614 & 1.930 & 2.648 & 3.068 & 2.821 & 3.123 \\
    8 &20 &30 & 1.344 & 1.454 & 1.388 & 1.432 & 1.443 & 1.480 & 1.519 & 1.546 & 1.535 & 1.888 & 2.304 & 2.954 & 3.212 & 2.602 \\
    8 &20 &40 & 1.353 & 1.449 & 1.406 & 1.448 & 1.471 & 1.531 & 1.572 & 1.543 & 1.486 & 2.153 & 2.637 & 3.179 & 2.871 & 3.160 \\
    8 &30 &10 & 1.363 & 1.575 & 1.967 & 3.650 & 3.596 & 1.325 & 1.331 & 1.464 & 1.800 & 2.663 & 2.515 & 3.068 & 3.159 & 3.246 \\
    8 &30 &20 & 1.352 & 1.561 & 1.741 & 1.346 & 1.372 & 1.348 & 1.380 & 1.479 & 1.675 & 2.817 & 2.845 & 3.255 & 2.860 & 3.146 \\
    8 &30 &30 & 1.359 & 1.667 & 1.364 & 1.370 & 1.391 & 1.376 & 1.386 & 1.481 & 1.714 & 2.630 & 3.091 & 2.839 & 2.821 & 2.319 \\
    8 &30 &40 & 1.362 & 1.524 & 1.391 & 1.387 & 1.412 & 1.418 & 1.426 & 1.461 & 2.132 & 2.462 & 2.935 & 3.154 & 2.999 & 2.846 \\
    8 &40 &10 & 1.334 & 1.737 & 3.011 & 2.304 & 3.266 & 1.351 & 1.294 & 1.455 & 1.798 & 3.186 & 3.180 & 3.146 & 3.012 & 3.214 \\
    8 &40 &20 & 1.345 & 1.675 & 1.903 & 1.365 & 1.366 & 1.350 & 1.357 & 1.444 & 1.963 & 2.390 & 3.207 & 2.806 & 2.686 & 3.165 \\
    8 &40 &30 & 1.358 & 1.742 & 1.370 & 1.364 & 1.360 & 1.363 & 1.386 & 1.568 & 1.932 & 2.518 & 3.040 & 3.026 & 3.297 & 3.288 \\
    8 &40 &40 & 1.346 & 1.872 & 1.364 & 1.364 & 1.366 & 1.382 & 1.359 & 1.423 & 2.165 & 2.979 & 2.968 & 3.236 & 2.826 & 3.203 \\
    \cmidrule(r){1-17}
    10 &10 &10 & 1.321 & 1.380 & 1.477 & 1.651 & 1.547 & 1.452 & 1.493 & 1.467 & 1.553 & 1.640 & 1.732 & 2.076 & 2.551 & 2.483 \\
    10 &10 &20 & 1.346 & 1.326 & 1.395 & 1.435 & 1.475 & 1.497 & 1.527 & 1.566 & 1.635 & 1.840 & 2.287 & 1.900 & 2.602 & 2.976 \\
    10 &10 &30 & 1.370 & 1.375 & 1.397 & 1.434 & 1.476 & 1.496 & 1.558 & 1.580 & 1.583 & 1.734 & 2.006 & 2.314 & 2.417 & 2.713 \\
    10 &10 &40 & 1.367 & 1.380 & 1.413 & 1.450 & 1.490 & 1.498 & 1.554 & 1.624 & 1.600 & 1.574 & 2.025 & 2.511 & 2.562 & 2.862 \\
    10 &20 &10 & 1.332 & 1.412 & 1.695 & 1.714 & 1.746 & 1.353 & 1.399 & 1.467 & 1.688 & 2.096 & 2.738 & 2.798 & 3.268 & 2.804 \\
    10 &20 &20 & 1.349 & 1.459 & 1.789 & 1.423 & 1.429 & 1.455 & 1.485 & 1.501 & 1.668 & 1.876 & 2.944 & 3.130 & 3.044 & 2.976 \\
    10 &20 &30 & 1.373 & 1.468 & 1.412 & 1.441 & 1.485 & 1.497 & 1.538 & 1.633 & 1.767 & 1.911 & 3.062 & 2.596 & 3.280 & 3.060 \\
    10 &20 &40 & 1.386 & 1.422 & 1.420 & 1.495 & 1.520 & 1.575 & 1.602 & 1.515 & 1.609 & 1.978 & 2.394 & 2.941 & 2.898 & 2.546 \\
    10 &30 &10 & 1.318 & 1.546 & 1.729 & 2.480 & 3.495 & 1.366 & 1.324 & 1.500 & 1.804 & 2.874 & 3.255 & 3.197 & 2.890 & 3.127 \\
    10 &30 &20 & 1.342 & 1.594 & 2.059 & 1.391 & 1.372 & 1.382 & 1.360 & 1.640 & 1.770 & 3.154 & 2.686 & 2.508 & 3.143 & 3.275 \\
    10 &30 &30 & 1.376 & 1.579 & 1.366 & 1.395 & 1.411 & 1.409 & 1.412 & 1.524 & 1.813 & 3.123 & 3.019 & 2.772 & 3.060 & 3.019 \\
    10 &30 &40 & 1.380 & 1.473 & 1.399 & 1.410 & 1.427 & 1.413 & 1.443 & 1.532 & 1.835 & 3.146 & 3.108 & 2.867 & 2.635 & 3.029 \\
    10 &40 &10 & 1.353 & 1.765 & 2.772 & 2.823 & 2.406 & 1.365 & 1.357 & 1.568 & 1.790 & 3.115 & 3.094 & 3.238 & 3.081 & 3.148 \\
    10 &40 &20 & 1.379 & 1.684 & 2.336 & 1.361 & 1.373 & 1.373 & 1.374 & 1.592 & 1.852 & 3.054 & 3.071 & 2.510 & 3.173 & 2.954 \\
    10 &40 &30 & 1.373 & 2.041 & 1.368 & 1.367 & 1.380 & 1.384 & 1.386 & 1.482 & 2.066 & 2.898 & 3.134 & 3.225 & 2.690 & 3.255 \\
    10 &40 &40 & 1.380 & 1.493 & 1.370 & 1.383 & 1.393 & 1.399 & 1.400 & 1.661 & 2.048 & 2.576 & 2.998 & 3.295 & 2.770 & 2.804 \\
    \bottomrule
\end{tabular}
\caption{The approximation ratios of our algorithm for $m=8$-$10$.}
\label{table:synthetic_sp2}
\end{table*}

\section{Full Experimental Results}

\textbf{Hardware description.} We performed our experiments on a machine having the following configuration: CPU: Intel(R) Xeon(R) CPU E5-2680 v4 @ 2.40GHz; Memory: 256 GB.

In this section, we illustrate the practical performance of our algorithms and study the significance of considering outliers for WB. Our experiments contain three parts. Firstly, we conduct the experiments on synthetic datasets, where the positions of the  barycenter supports are predefined, allowing us to compute the exact optimal objective value for measuring the approximation ratio of our algorithm. Secondly, 
%given the relation between k-sparse WB with outliers and fairness clustering with outliers (as discussed Claim~\ref{cl-m2}),
we compare our algorithms with several baselines 
%including the existing fairness clustering algorithm
on real-world datasets. Finally, we provide the visualized results  on the MNIST dataset~\cite{lecun2010mnist}. 
%Due to the space limit, the full experimental results are placed to our supplement. 

% \paragraph{Titled paragraphs.} You should use titled paragraphs if and only if the title covers exactly one paragraph. Such paragraphs should be separated from the preceding content by at least 3pt, and no more than 6pt. The title should be in 10pt bold font and to end with a period. After that, a 1em horizontal space should follow the title before the paragraph's text.

\textbf{Datasets.} In our synthetic datasets, we set the supports size $k\in[10,40]$ and the dimensionality $d\in[10,40]$; each instance comprises $m \in [2,10]$ different distributions, where each distribution consists of $n=20,000$ points.  The true barycenter supports are uniformly sampled within a hypercube with a side length of 10, and random weights are assigned to each center point. Points are randomly generated within Gaussian balls around the centers based on the assigned weights. We introduce outliers by uniformly sampling $z$ points for each distribution within the cube, with $z$ ranging from $0$ to $0.15\times n$. 

We also select three widely-used datasets from the UCI repository~\cite{UCI}: 
\textbf{Bank}~\cite{bank} ($4,521$ points in $\mathbb{R}^3$) represents the individual telephone calls during a marketing campaign, which
contains the information of the customers.
We have $m=3$  distributions categorized  based on marital status. 
\textbf{Credit card}~\cite{credit_card} ($30,000$ points in $\mathbb{R}^{14}$) includes the information about the credit card holders. We 
partitioned the data into $m=9$ distributions based on marriage and education. 
\textbf{Adult}~\cite{adult} ($32,561$ points in $\mathbb{R}^5$) represents the individual information from the 1994 U.S. Census. We 
partitioned it into $m=10$ distributions based on sex and race.  Finally, $5\%$ random noise are added to each dataset as outliers.

\textbf{Baselines and our implementation.} It is worth noting that there is no method that explicitly addresses $k$-sparse WB with outliers or fair clustering with outliers, to the best of our knowledge. We employ three baselines. First, following Remark~\ref{rem-fixoutlier}, we consider the fixed-support WB with outliers algorithm, utilizing $k$ random centers as support, and compute the optimal weight distribution via LP (denoted as ``Random\_$\mathcal{O}$'').
The other two baselines include a fair clustering algorithm that does not consider outliers (denoted as ``FC\_$\mathcal{O}$'')~\cite{bera2019fair}, and a non-fair clustering method considering outliers ``$k$-means-\,-\_$\mathcal{O}$''~\cite{chawla2013k}.
For the FC algorithm, we identify the farthest points in each class as outliers; for $k$-means-\,-, after obtaining the support positions, a new fair clustering solution can be obtained through LP. Additionally, we also test their three ``plain'' versions that do not discard  outliers, aiming to study the significance of considering outliers (denoted as ``Random'', ``FC'', ``$k$-means-\,-'', respectively).
 
 %Despite our algorithm being bi-criteria, we strictly set the support size as $k$ and exclude the outliers with total weight exactly $z$. 
% test

In our implementation, we use the $k$-means++\cite{Arthur2007kmeansTA} as Algorithm $\mathcal{A}$ and $k$-means-\,-\cite{chawla2013k} as Algorithm $\mathcal{B}$; we also employ the LP solver~\cite{gurobi}  as the subroutine for solving fixed-support WB with outliers. To ensure a fair comparison, although Theorem~\ref{the-result1-2} suggests removing $2z$ outliers, we  remove only $z$ outliers in reality. Also, to keep $k$-sparsity for the result returned by $\mathcal{A}$, we only retain the top $k$ centers with the largest cluster sizes. We use ``Our\_$\mathcal{A}$'' and ``Our\_$\mathcal{B}$'' to denote them.

%\subsection{Results}
\vspace{0.05in}
\textbf{Results on synthetic datasets.} We compute the optimal cost $\mathtt{Cost}_{\mathtt{opt}}$ for the WB problem by using the pre-specified barycenter support. Subsequently, we execute our algorithm under various parameters to obtain the $\mathtt{Cost}$ and calculate the approximation ratio, defined as $\frac{\mathtt{Cost}}{\mathtt{Cost}_{\mathtt{opt}}}$. The results obtained by Algorithm Our\_$\mathcal{A}$ and Our\_$\mathcal{B}$ is presented in Table \ref{table:synthetic_sp1} and Table \ref{table:synthetic_sp2}. We can see that our algorithm consistently achieves favorable approximation ratios across different dimensions and outlier proportions.

\textbf{Results on real datasets.} 
%We execute our algorithms with $\mathcal{A}$ and $\mathcal{B}$, and the three baseline algorithms,  on the real-world datasets.
%We ensured that these algorithms obtained $k$ supports and excluded $5\%$ of outliers, ultimately yielding respective costs. 
The results are illustrated in Figure \ref{fig:real_sup}. 
As can be seen, even with only $5\%$ outliers, the plain versions of the three baselines take almost double costs than their counterparts who consider outliers. Moreover, our algorithms demonstrate even lower costs across all the datasets with different values of $k$.

\begin{figure}[H]
  \centering
  \includegraphics[width=0.48\textwidth]{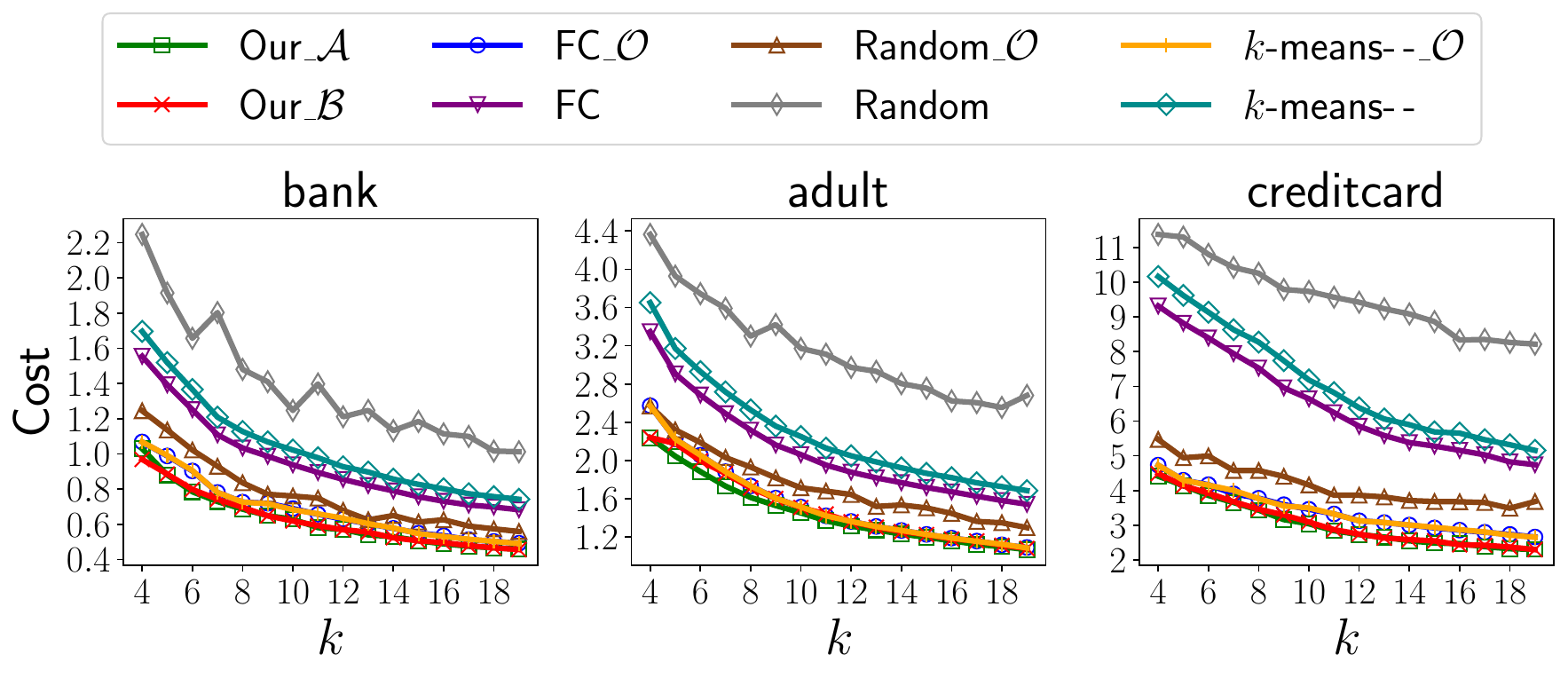}
  \caption{The obtained costs on real datasets.}
  \label{fig:real_sup}
\end{figure}

\textbf{Visualized results.} 
In Figure \ref{fig:mnist_sup} and Figure \ref{fig:mnist_sup2}, we show the $40$-sparse barycenters obtained by Our\_$\mathcal{A}$ and Our\_$\mathcal{B}$ for digit $0$-$9$ in the MNIST dataset, with $2\%$ of outliers removed from each digit. It is evident that the obtained set of $40$ points effectively captures the distinctive features for each digit.

\begin{figure}[H]
  \centering
  \includegraphics[width=0.45\textwidth]{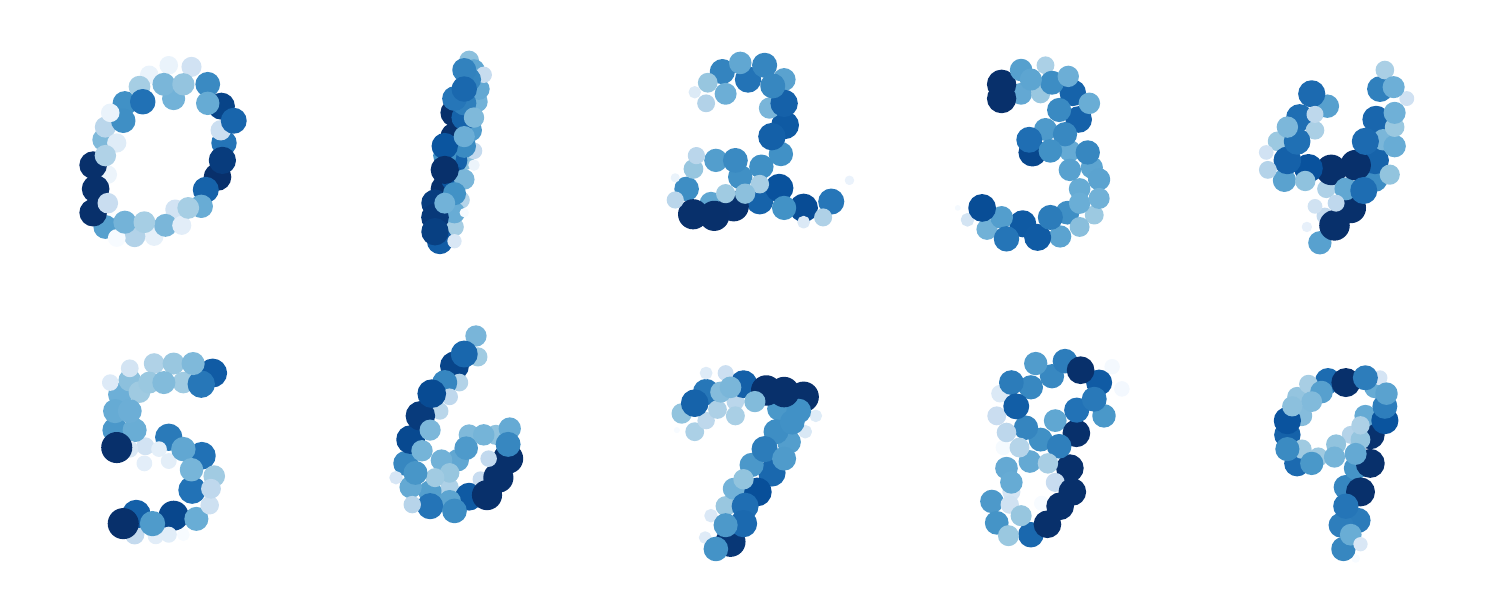}
  \caption{$k$-sparse WB obtained by Our\_$\mathcal{A}$ for $k=40$.}
  \label{fig:mnist_sup}
\end{figure}

\begin{figure}[H]
  \centering
  \includegraphics[width=0.45\textwidth]{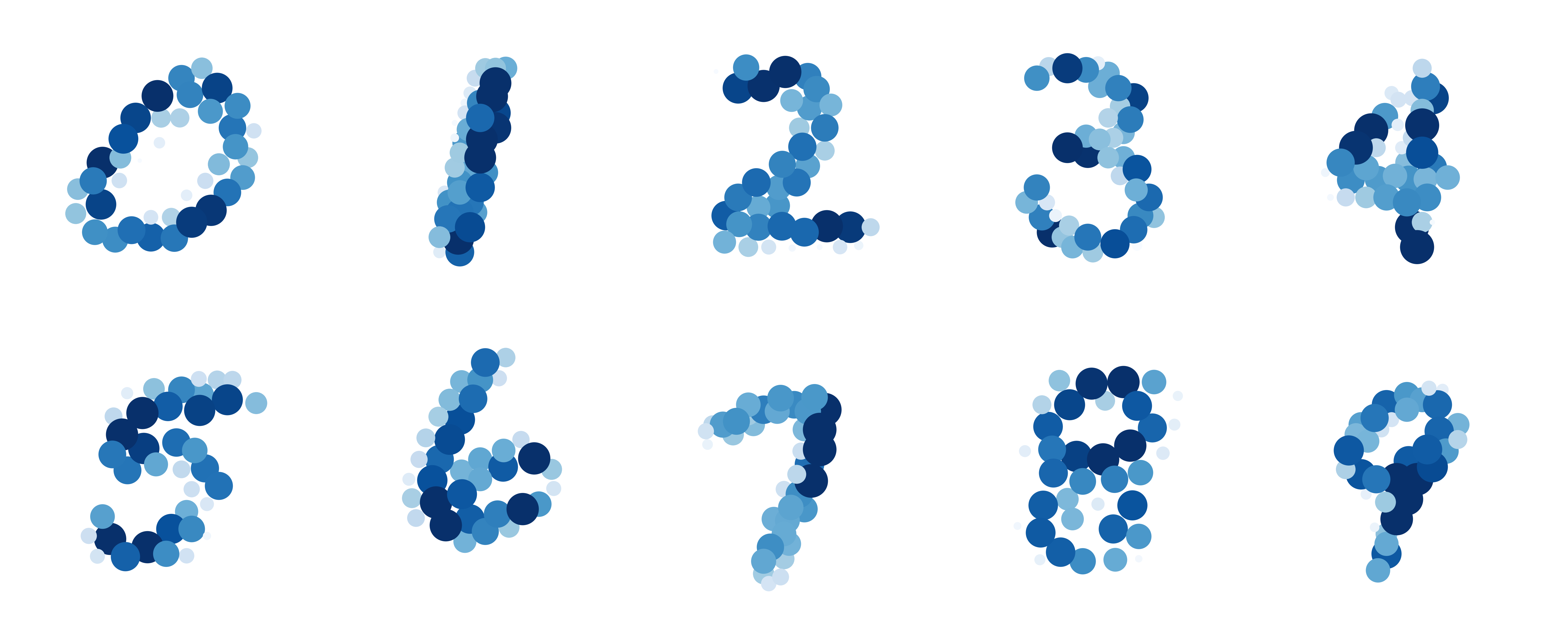}
  \caption{$k$-sparse WB obtained by Our\_$\mathcal{B}$ for $k=40$.}
  \label{fig:mnist_sup2}
\end{figure}

\end{document}